\documentclass[twoside,11pt]{article}

\usepackage[abbrvbib, preprint]{jmlr2e}

\usepackage{commath,mathtools}
\usepackage{algorithm,algorithmic}
\usepackage{subcaption}
\usepackage{color}
\usepackage{rotating}

\usepackage{commath,mathtools}
\usepackage{algorithm,algorithmic}
\usepackage{subcaption}
\usepackage{rotating}

\usepackage{booktabs}
\usepackage{siunitx}

\jmlrheading{1}{2021}{1-48}{4/00}{10/00}{meila00a}{Shaojun Ma, Haodong Sun, Xiaojing Ye, Hongyuan Zha and Haomin Zhou}

\ShortHeadings{Learning Cost Function for Optimal Transport}{Ma, Sun, Ye, Zha, and Zhou}
\firstpageno{1}

\DeclareMathOperator*{\argmin}{\mathrm{arg\,min}}

\DeclareMathOperator{\prox}{\mathrm{prox}}

\newcommand{\Ehat}{\hat{\mathbb{E}}}
\newcommand{\Etd}{\widetilde{E}}
\newcommand{\Dcal}{\mathcal{D}}
\newcommand{\Wcal}{\mathcal{W}}
\newcommand{\Wcaltd}{\widetilde{\mathcal{W}}}
\newcommand{\pidata}{\hat{\pi}}
\newcommand{\KL}{\text{KL}}

\begin{document}

\title{Learning Cost Function for Optimal Transport}

\author{\name Shaojun Ma \email shaojunma@gatech.edu \\
       \addr School of Mathematics\\
       Georgia Institute of Technology\\
       Atlanta, GA 30332, USA
       \AND
       \name Haodong Sun \email hdsun@gatech.edu \\
       \addr School of Mathematics\\
       Georgia Institute of Technology\\
       Atlanta, GA 30332, USA
       \AND
       \name Xiaojing Ye \email xye@gsu.edu \\
       \addr Corresponding author. Department of Mathematics and Statistics\\
       Georgia State University\\
       Atlanta, GA 30303, USA
       \AND
       \name Hongyuan Zha \email zhahy@cuhk.edu.cn \\
       \addr School of Data Science, Shenzhen Research Institute of Big Data\\
       The Chinese University of Hong Kong\\
       Shenzhen, Guangdong 518172, China
       \AND
       \name Haomin Zhou \email hmzhou@gatech.edu \\
       \addr School of Mathematics\\
       Georgia Institute of Technology\\
       Atlanta, GA 30332, USA}

\editor{}

\maketitle

\begin{abstract}
Inverse optimal transport (OT) refers to the problem of learning the cost function for OT from observed transport plan or its samples. In this paper, we derive an unconstrained convex optimization formulation of the inverse OT problem, which can be further augmented by any customizable regularization. We provide a comprehensive characterization of the properties of inverse OT, including uniqueness of solutions. We also develop two numerical algorithms, one is a fast matrix scaling method based on the Sinkhorn-Knopp algorithm for discrete OT, and the other one is a learning based algorithm that parameterizes the cost function as a deep neural network for continuous OT. The novel framework proposed in the work avoids repeatedly solving a forward OT in each iteration which has been a thorny computational bottleneck for the bi-level optimization in existing inverse OT approaches. Numerical results demonstrate promising efficiency and accuracy advantages of the proposed algorithms over existing state-of-the-art methods.
\end{abstract}


\section{Introduction}
\label{sec:intro}

\paragraph{Optimal transport} Optimal transport (OT) is a mathematical subject involving a wide range of areas, including differential geometry, partial differential equations, optimization, and probability theory \citep{ambrosio2003optimal,villani2008optimal}. It has observed many applications in machine learning in recent years \citep{peyre2019computational}.
Let $X$ and $Y$ be two measure spaces, $\mu \in P(X)$ be a probability measure on $X$ and $\nu \in P(Y)$ a probability on $Y$, and $c(x,y): X \times Y \to \mathbb{R}$ be a cost function (also known as the ground metric, although it needs not be a metric in the standard sense) that quantifies the effort of moving one unit of mass from location $x$ to location $y$.
Then the problem of optimal transport aims at finding the joint probability measure $\pi^* \in P(X \times Y)$, with marginal distributions $\mu$ and $\nu$, that minimizes the total cost. More precisely, $\pi^*$ solves the following constrained minimization problem:
\begin{equation}\label{eq:ot}
\min_{\pi \in \Pi(\mu,\nu)}  \int_{X \times Y} c(x,y) \; d \pi(x,y),
\end{equation}
where $\Pi(\mu, \nu)$ is the set of probability measures $\pi \in P(X\times Y)$ such that
\begin{equation}
    \label{eq:Pi}
    \pi(A \times Y) = \mu(A) \quad \mbox{and}\quad \pi(X\times B) = \nu(B)
\end{equation}
for any measurable set $A$ of $X$ and $B$ of $Y$.
A solution $\pi^*$ to \eqref{eq:ot} is called an optimal transport plan between $\mu$ and $\nu$ under $c$.
%
%
The optimal function value of \eqref{eq:ot} is also called the earth mover's distance (EMD) between the two distributions $\mu$ and $\nu$, suggesting the minimal total effort to move a pile of earth shaped as $\mu$ to form the pile shaped as $\nu$.

The discrete version of \eqref{eq:ot} reduces to a linear program. In this case, $\mu \in \Delta^{m-1}$ and $\nu \in \Delta^{n-1}$ become two probability vectors, where $\Delta^{n-1}:=\{x = (x_1,\dots,x_n) \in \mathbb{R}^n: x_i\ge 0,\ \sum_{i=1}^n x_i = 1\}$ stands for the standard probability simplex in $\mathbb{R}^n$, and the cost $c$ and transport plan $\pi$ each renders an $m \times n$ matrix.
Then the discrete OT problem reads
\begin{equation}\label{eq:dis_ot}
\min_{\pi \in \mathbb{R}^{m\times n}} \cbr[2]{\langle c,\pi \rangle \, : \, \pi\ge 0,\ \pi 1_n = \mu, \ \pi^{\top}1_m = \nu},
\end{equation}
where $\langle c,\pi \rangle:=\text{tr}(c^{\top}\pi)=\sum_{i,j}c_{ij}\pi_{ij}$ is the discretized total cost under transport plan $\pi$, and $1_n:=[1,\dots,1]^{\top}\in\mathbb{R}^n$. We will use these matrix and vector notations exclusively in Sections \ref{subsec:dis_iot}, \ref{subsec:dis_iot_synthetic}, and \ref{subsec:dis_iot_marriage} during the discussions on discrete inverse OT.

It is important to note that a brute-force discretization of \eqref{eq:ot} yields in general a computationally intractable problem \eqref{eq:dis_ot} when $X$ and $Y$ are in higher dimensional space:
$m$ represents the number of grid points (or bins) for discretizing the space $X \subset \mathbb{R}^d$, and hence it grows exponentially fast in $d$, i.e., $m=N^d$ where $N$ is the discretization points (resolution) in each dimension. Similar for $n$ and $Y$. This is well known as the ``curse of dimensionality'' in the literature. Therefore, continuous OT and its inverse problem require a vastly different approach in analysis and computation.

The past two decades have witnessed substantial developments in the theory and computation of OT.
Moreover, OT gained significant interests in the machine learning community in the past few years, and has been employed in a variety of applications, such as domain adaption \citep{courty2016optimal}, regularization \citep{shafieezadeh2017regularization}, parameter estimation \citep{dessein2017parameter}, dictionary learning \citep{rolet2016fast}, Kalman filtering \citep{abadeh2018wasserstein}, image processing \citep{papadakis2015optimal}, graph partition \citep{abrishami2019geometry}, information geometry \citep{amari2018information}, among many others.

\paragraph{Motivation} 
In this paper, we consider the inverse problem of OT, i.e., learning the cost function $c$ from observations of the joint distribution $\pi^*$ or its samples.
This work is motivated by a critical issue of OT in all real-world applications: the solution to OT heavily depends on the cost function $c$, which is the sole latent variable in the OT problem \eqref{eq:ot} to deduce the optimal transport plan for any give pair of marginal distributions $(\mu,\nu)$. Therefore, the cost function is of paramount importance in shaping the optimal transport plan $\pi^*$ for further analysis and inference.

In most existing applications of OT, the cost function is simply chosen as a distance-like functions, such as $c(x,y) = \|x-y\|^p$ where $p>0$, according to users' preferences.
However, there are infinitely many choices of $p$, and cost functions need not be even distance-like in practice. Therefore, a user-chosen cost function may incorrectly estimate the cost incurred to transfer probability masses and thus fail to capture the actual underlying structures and properties of the data. Eventually, a mis-specified cost function results in severely biased optimal transport plan given new marginal distribution pairs, leading to false claims and invalid inferences. 

To address the aforementioned issue, we propose to leverage observed pairing data available in practice, which are samples or realizations of the optimal transport plan, to reconstruct the underlying cost function \citep{cuturi2014ground,dupuy2014personality,galichon2015cupids,li2019learning,stuart2019inverse}. 
The reconstructed cost function can be used to study the underlying mechanism of transferring probability/population mass in the application of interests. It can also be used to estimate optimal pairings given new marginal distributions $\mu$ and $\nu$ where modeling and computation of OT are involved.
%
%

\paragraph{Approach}
We consider the inverse problem of entropy regularized OT (the reason of using regularized OT over the unregularized counterpart is explained in Section \ref{subsec:entropy_iot}). 
We propose a variational formulation for learning the cost function such that its induced optimal transport plan is close to the observed transport plan or its samples.
This variational formulation yields a bi-level optimization problem, which can be challenging to solve in general.
However, by leveraging the dual form of entropy regularized OT, we show that this bi-level optimization can be reformulated as an unconstrained and convex problem in the cost function before adding any customizable regularization.

Based on our new formulation, we develop two efficient numerical algorithms for inverse OT: one for the discrete case and the other for continuous case.
In the discrete case, we can realize the observed transport plan $\pidata$ as a probability matrix, which is either directly given or can be readily summarized using samples. 
Then we show that the cost $c$ can be computed using a fast matrix scaling algorithm.
In the continuous case, $\pidata$ is often presented by a number of i.i.d.~samples in the form of $(x,y) \sim \pidata$.
In this case, we parameterize the cost function as a deep neural network and develop a learning algorithm that is completely mesh-free and thus capable of handling high dimensional continuous inverse OT problems.

A significant advantage of our approach over existing ones is that we can avoid solving a standard OT problem in each iteration in bi-level optimization. To better distinguish from the inverse OT, we hereafter use the term \emph{forward OT} to refer the standard OT problem which solves for the optimal transport plan $\pi^*$ given cost $c$.
Thus, \emph{the computational complexity our method is comparable to that of a forward OT problem}, which is only a small fraction of complexities of existing bi-level optimization based methods.
We will demonstrate this substantial improvement in efficiency in Section \ref{sec:experiment}.

\paragraph{Novelty and contributions}
Existing approaches to the inverse problem of OT aim at recovering the cost function but vary in specific problem formulations and applications domains. These approaches will be discussed in more details in Section \ref{sec:literature}. Compared to existing ones, our approach is novel and advantageous in the following aspects:
\begin{itemize}
    \item
    All existing methods formulate the cost learning as bi-level optimization or its variants, which require \emph{solving the forward OT problem in each outer iteration}. In sharp contrast, our approach yields a convex optimization with customizable regularization, and the complexity of our method is \emph{comparable to the complexity of one forward OT}.
    \item 
    To the best of our knowledge, the present work is the first in the literature that provides a \emph{comprehensive characterization of the solution(s) of inverse OT}. Moreover, we show that the ill-posedness of inverse OT, particularly the ambiguity issue of unknown cost functions, can be rectified and the ground truth cost function can be recovered robustly under mild conditions.
    \item
    Our framework can be applied to \emph{both discrete and continuous settings}. To our best knowledge, the present work is the \emph{first to tackle cost function learning for continuous OT}. This enables the application of inverse OT in a large variety of real-world problems involving high-dimensional data.
\end{itemize}
%

\paragraph{Organization} 
The remainder of this paper is organized as follows.
We first provide an overview of existing cost learning approaches in OT and the relations to other metric learning problems in Section \ref{sec:literature}.
In Section \ref{sec:framework}, we propose an inverse optimal transport approach for cost function learning, and derive a novel framework based on the dual of the inverse OT formulation.
In Section \ref{sec:algorithm}, we develop two prototype algorithms to recover the cost matrix in the discrete setting and the cost function in continuous setting, and discuss their properties and variations.
Numerical experiments and comparisons are provided in Section \ref{sec:experiment}. 
Section \ref{sec:conclusion} concludes this paper.

\section{Related Work}
\label{sec:literature}
In this section, we provide an overview of OT, inverse OT, and several closely related topics. We also show the relations between cost learning for OT and general metric learning, and contrast our approach to the existing methods.

\paragraph{Computational OT}
The computation of OT has been a long standing challenge and is still under active research.
Most existing work focus on the discrete setting \eqref{eq:dis_ot}, which is a special type of linear program (LP).
However, the cubic computation complexity for general LP solvers prohibits fast numerical solution for large $m$ and $n$.
In \citep{cuturi2013sinkhorn}, a modification of \eqref{eq:dis_ot} with an additional entropy regularization term in the objective function is proposed:
\begin{equation}
\label{eq:dis_rot}
\min_{\pi \in \mathbb{R}^{m\times n}} \cbr[2]{\langle c,\pi \rangle - \varepsilon H(\pi)\, :\, \pi 1_n = \nu, \ \pi^{\top}1_m = \nu},
\end{equation}
where $H(\pi):=-\langle \pi, \log \pi - 1 \rangle = -\sum_{i,j}\pi_{ij}(\log \pi_{ij} - 1)$ is the (normalized) Shannon entropy of $\pi$, and $\varepsilon>0$ is a prescribed weight of the entropy regularization.
Due to the entropy term, the troublesome inequality constraint in the original OT \eqref{eq:dis_ot} is eliminated, and the objective function in \eqref{eq:dis_rot} becomes strictly convex which admits unique solution.
Moreover, the dual problem of \eqref{eq:dis_rot} is unconstrained, which can be solved by a fast matrix scaling algorithm called the Sinkhorn (or Sinkhorn-Knopp) algorithm \citep{cuturi2013sinkhorn}.
Sinkhorn algorithm has been the common approach to solve (regularized) OT \eqref{eq:dis_rot} numerically in the discrete setting since then.
Its property, convergence, and relation to the original OT \eqref{eq:dis_ot} are also extensively studied, for instance, in \citep{dessein2018regularized,schmitzer2019stabilized}.
A more comprehensive treatment of computational OT in discrete setting, especially in the regularized form, can be found in \citep{peyre2019computational}.
The continuous OT problem is considered where the dual variables are parameterized as deep neural networks \citep{seguy2018large}.
%
The sample complexity of OT is also studied in \citep{genevay2018sample}.

\paragraph{Cost Learning for OT} 
The problem of cost learning for optimal transport has received considerable attention in the past few years.
In \citep{dupuy2014personality,galichon2015cupids}, the cost matrix is parameterized as a bilinear function of the feature vectors of the two sides in optimal transport.
The parameter of the bilinear function, i.e., the interaction matrix, is recovered from the observed matchings, which are hypothesized to be samples drawn from the optimal transport plan that maximizes the total social surpluses \citep{dupuy2014personality}.
The interaction matrix quantifies coupling surplus that are important in the study of quantitative economics.
In \citep{li2019learning}, a primal-dual matrix learning algorithm is proposed to allow more flexible parametrization of the cost matrix and also takes into account inaccurate marginal information for robust learning.
In \citep{cuturi2014ground}, a set of distributions are given where each pair is also associated with a weight coefficient, and the cost matrix is learned by minimizing the weighted sum of EMD 
between these pairs induced by this cost.
Given class labels of documents which are represented as histograms of words, the cost matrix is parameterized as Mahalanobis distance between feature vectors of the words and learned such that the induced EMD 
between similar documents are small \citep{huang2016supervised}.
In \citep{wang2012supervised}, the cost matrix is learned such that the induced EMD 
between histograms labeled as similar are separated from those between dissimilar histograms, which mimics the widely used metric learning setup.
In \citep{zhao2018label}, the cost matrix is induced by a kernel mapping, which is jointly learned with a feature-to-label mapping in a label distribution learning framework.
This work is extended to a multi-modal, multi-instance, and multi-label learning problem in a follow-up work \citep{yang2018complex}.
In \citep{xu2019learning}, the cost matrix is parameterized as the exponential of negative squared distance between features, where the feature map is learned such that the induced EMD is small for those with same labels and large otherwise.
A cost matrix learning method based on the Metropolis-Hasting sampling algorithm is proposed in \citep{stuart2019inverse}.
In \citep{liu2019learning}, the Sinkhorn iteration is unrolled into a deep neural network with cost matrix as unknown parameter, which is then trained using given side information.
The present work targets at the cost learning problem for OT as in the aforementioned ones, but contrasts favorably to them as explained in Section \ref{sec:intro}.

\paragraph{General Metric Learning} 
The aforementioned methods and the work presented in this paper aim at learning the cost matrix/function, which is related to but different from the standard metric learning \citep{bellet2013survey} in machine learning.
In standard metric learning, the goal is to directly learn the distance that quantifies the similarity between features or data points given in the samples.
In contrast, the learning problem in inverse OT aims at recovering the cost function (also known as the ground metric) that induces the EMD 
(or more generally the Wasserstein distance) and optimal transport plan, optimal couplings, or optimal matchings exhibited by the data.
In inverse OT, we only observe the couplings/matchings which are not labeled as similar or not.
Hence, we cannot directly assess the distance or cost between features.
Instead, we need to learn the cost based on the relative frequency of the matchings in the observed data, which is a compounded effect of the cost function and the intra-population competitions.
Moreover, the cost function is critical to reveal the underlying mechanism of optimal transport and matchings, and can be used to predict or recommend optimal matchings given new but different marginal distributions \citep{dupuy2014personality,galichon2015cupids,li2019learning}.

\paragraph{Riemann Distance Learning}
The cost function learning problem for optimal transport is also related to Riemannian metric learning on probability simplex.
In \citep{dagnew2015supervised,le2015unsupervised}, the Riemannian metric, i.e., distances between probability distributions or histograms on the manifold of probablity simplex, is directly learned.
In contrast, the goal in this work is to learn the cost function that reveals the interaction between features, which can also automatically induce a metric on the probability simplex if the cost function satisfies proper conditions.
Moreover, the learned cost function from inverse OT can be used to provide insights of observed matchings and generate interpretable predictions on new data associations, which are extremely important and useful in many real-world applications.

\vspace{-6pt}
\section{Proposed Framework}
\label{sec:framework}

\subsection{Preliminaries on Entropy Regularized OT}
\label{subsec:preliminary}
We consider the entropy regularized OT \citep{cuturi2013sinkhorn,peyre2019computational,amari2019information,aude2016stochastic,dessein2018regularized,dvurechensky2018computational,dvurechensky2020stable,genevay2019entropy,janati2020entropic,paty2020regularized}, where the objective function in \eqref{eq:ot} is supplemented by the (negative) entropy of the unknown distribution $\pi$.
Entropy regularization takes into account of the uncertainty and incompleteness of observed data, which an important advantage over the OT without regularization \citep{cuturi2014ground,dupuy2014personality,galichon2015cupids,peyre2019computational}.
The entropy regularized forward OT problem \eqref{eq:dis_rot} is given as follows:
\begin{equation}
\label{eq:rot}
\min_{\pi \in \Pi(\mu,\nu)} \cbr[3]{\int_{X \times Y} c(x,y)\,d\pi(x,y) - \varepsilon H(\pi)},
\end{equation}
where $H(\pi) := -\int  (\log (d \pi / d\lambda) + 1) \, d \pi$ denotes the normalized entropy of $\pi$ (we assume $\pi$ is absolutely continuous with respect to the Lebesgue measure $\lambda$ of $X\times Y$ and $d \pi / d \lambda $ denotes the Radon-Nikodym derivative.)
%
%
Let $\alpha:X\to\mathbb{R}$ and $\beta:Y\to\mathbb{R}$ be the Lagrangian multipliers corresponding to the two marginal constraints in $\Pi(\mu,\nu)$ in \eqref{eq:Pi} respectively, we obtain the dual problem of \eqref{eq:rot} as follows,
\begin{equation}
\label{eq:rot_dual}
\max_{\alpha,\beta} \cbr[3]{ \int_X \alpha \,d\mu + \int_Y \beta \,d\nu  - \varepsilon \int_{X \times Y} e^{(\alpha+\beta-c)/\varepsilon} \,d\lambda }.
\end{equation}
Denote $(\alpha^c, \beta^c)$ the optimal solution of the dual problem \eqref{eq:rot_dual} for the given cost function $c$, we can readily deduce that the optimal solution $\pi^c$ to the primal problem \eqref{eq:rot} reads
\begin{equation}
\label{eq:rot_primal_var}
d\pi^c(x,y) = e^{(\alpha^c(x)+\beta^c(y)-c(x,y))/\varepsilon} \,d\lambda,
\end{equation}
which is a closed-form expression of $\pi^c$ in terms of $(\alpha^c, \beta^c)$. For notation simplicity, we omit the arguments $x$ and $y$ hereafter when there is no danger of confusion.

\subsection{Entropy Regularization in Inverse OT}
\label{subsec:entropy_iot}
Entropy regularization is particularly important to properly define the inverse problem of OT. 
To see this, we first consider the discrete OT problem \eqref{eq:dis_ot}. Notice that an observation matrix $\pidata$ containing zero entries does not provide necessary information to fully characterize the cost matrix $c$: a small $\pidata_{ij}$ suggests that $c_{ij}$ is relatively large; but if $\pidata_{ij} = \pidata_{ik}= 0$ then it is difficult to tell which of $c_{ij}$ and $c_{ik}$ is larger. Although this issue can be somewhat mitigated with additional information on $c$, it is still a severe problem when $\pidata$ contains many zeros or such information on $c$ is not available. Indeed, a reasonable inverse OT algorithm needs the relative ratios between the entries (thus better not be zeros) of $\pidata$, together with the supply distribution given by $\mu$ and $\nu$, to infer the underlying cost accurately. However, sparse $\pidata$ is very common in discrete OT problem \eqref{eq:dis_ot} without any regularization. This is because that an optimal transport plan, as a solution to the linear program \eqref{eq:dis_ot}, often occurs at an extremal point of the polytope of the contraint set, and thus the number of nonzero entries is no more than $m+n-1$ \citep{peyre2019computational}. In this case, the chance to uncover the true cost $c$ is very low.
Entropy regularization of OT overcomes this issue as it always yield a transport plan $\pidata$ with no zero entry. Moreover, as the weight $\varepsilon$ approaches $0$, the solution of entropy regularized OT tends to that of the standard unregularized OT.

Continuous inverse OT also benefits from entropy regularization. As shown later, our approach is based on the dual formulation of entropy regularized OT. This allows us to parameterize the dual variables and the cost function as deep neural networks and develop an efficient mesh-free method suitable for applications in high-dimensional continuous spaces. 

Due to the aforementioned reasons, we consider the inverse problem of entropy regularized OT in the present work. More precisely, we assume that the observation $\pidata$ was the solution of an entropy regularized OT \eqref{eq:rot} with unknown cost $c$ to be recovered.
However, it is important to note that, unlike entropy regularized OT, the inverse problem is \emph{not} sensitive to the weight $\varepsilon$ in \eqref{eq:rot}. This is because that the solution to \eqref{eq:rot} only depends on $c/\varepsilon$, rather than the actual $c$. Thus the observation $\pidata$ only contains information of this ratio $c/\varepsilon$ which is all we can recover. This ratio $c/\varepsilon$ provides all information needed no matter which OT one prefers to use for inference and prediction later: the solution to an unregularized OT is invariant to any constant scaling of $c$; and the solution to an entropy regularized OT can be freely modified by tuning another user-chosen regularization weight $\varepsilon'>0$ pretending that the given cost is just $c/\varepsilon$.

\subsection{Inverse OT and Its Dual Formulation}
\label{subsec:iot-formulation}
%
%
Suppose that the marginal distributions are given as $\mu$ and $\nu$, and we observed sample transport plan $\pidata \in \Pi(\mu,\nu)$ (details about the format of $\pidata$ will be provided in the next section).
Then we propose the following inverse OT problem to learn the underlying cost $c$ from observation $\pidata$:
\begin{subequations}
\label{eq:irot}
\begin{align}
\min_{c} &\quad \KL(\pidata, \pi^c ) + \varepsilon^{-1} R(c), \label{eq:irot_upper} \\
\text{s.t.} &\quad \pi^c = \argmin_{\pi \in \Pi(\mu,\nu)} \cbr[2]{\int_{X \times Y} c\,d\pi - \varepsilon H(\pi)}, \label{eq:irot_lower} 
\end{align}
\end{subequations}
where the Kullback-Leibler (KL) divergence between $\pidata$ and $\pi$ is defined by
\begin{equation}
    \KL(\pidata, \pi) := \int_{X \times Y}  \frac{d \pidata}{d \pi} \log \del[2]{\frac{d \pidata}{d \pi}} \,d\pi,
\end{equation}
$d\pidata/d\pi$ is the Radon-Nikodym derivative of $\pidata$ with respect to $\pi$, $R(c)$ represents the regularization (or constraint) on the cost function $c$, which is to be specified later, and $\pi^c$ is the optimal transport plan induced by $c$, i.e., the solution of \eqref{eq:rot} for any specific $c$. We multiplied $\varepsilon^{-1}$ to $R(c)$ in \eqref{eq:irot_upper} for notation simplicity later, but $R(c)$ is user-defined and thus can contain $\varepsilon$ for cancellation.
The model \eqref{eq:irot} is straightforward to interpret: we seek for the cost function $c$ such that the induced $\pi^c$ is close to the observed transport plan $\pidata$ in the sense of KL divergence, and meanwhile it respects the specified regularization or satisfies the constraint described by $R(c)$.

The problem \eqref{eq:irot} is a typical bi-level optimization: the upper level problem \eqref{eq:irot_upper} involves $\pi^c$ which is the solution to the minimization in the lower level problem \eqref{eq:irot_lower}.
In general, bi-level optimization problems such as \eqref{eq:irot} are considered very challenging to solve: standard bi-level optimization methods require solving the lower level problem \eqref{eq:irot_lower} during \emph{each update of the variable $c$}. Therefore, the overall computational cost is very high because the lower level problem, which is an expensive OT problem, needs to be solved for many times (i.e., the number of outer iterations to update $c$, which can be easily over hundreds or even thousands). 

However, we show that the inverse OT problem \eqref{eq:irot} possesses a very special structure. Most notably, we prove that it is equivalent to an unconstrained and convex problem in $c$ before adding $R(c)$ by leveraging the dual form of the entropy regularized OT \eqref{eq:rot_dual}. The rigorous statement of this equivalency relation is given in the following theorem.
\begin{theorem}
The bi-level optimization \eqref{eq:irot} for inverse OT is equivalent to
\begin{align}
\label{eq:irot_final}
\min_{\alpha,\beta,c}\ E(\alpha,\beta,c) + R(c),
\end{align}
where the functional $E$ is defined by
\begin{align}
E(\alpha,\beta,c) := \int_{X \times Y} c \,d\pidata - \int_X \alpha \,d\mu - \int_Y \beta \,d\nu  + \varepsilon \int_{X \times Y} e^{(\alpha+\beta-c)/\varepsilon} \,d\lambda. \label{eq:def_E}
\end{align}
The equivalency relation is in the sense that $c^*$ solves \eqref{eq:irot} if and only if $(\alpha^{c^*},\beta^{c^*},c^*)$ solves \eqref{eq:irot_final}, where $(\alpha^{c^*},\beta^{c^*})$ stands for the optimal solution to the dual problem \eqref{eq:rot_dual} with cost $c^*$.
\end{theorem}
\begin{proof}
Recall that the optimal transport plan $\pi^c$ induced by $c$ is given in \eqref{eq:rot_primal_var}, where $(\alpha^c,\beta^c)$ is the optimal solution to the dual problem \eqref{eq:rot_dual}. Therefore, \eqref{eq:rot_primal_var} implies $d\pi^c = \rho\, d\lambda$ where $\rho$ is the density function of $\pi^c$ given by
\begin{equation}
\label{eq:rho}
\rho(x,y) := e^{(\alpha^c(x)+\beta^c(y)-c(x,y))/\varepsilon}.
\end{equation}
Since $\rho>0$ everywhere, we know $\lambda \ll \pi^c$ and $d\lambda = \rho^{-1} d \pi^c$. On the other hand, since $\pidata \ll \pi^c \ll \lambda$, we know $d \pidata = \hat{\rho} \,d\lambda$ for some probability density $\hat{\rho}$. Hence, by the chain rule of measures, we have
\begin{align}
    \KL(\pidata, \pi^c)
    & = \int_{X \times Y}  \frac{\hat{\rho}}{\rho} \log \del[2]{\frac{\hat{\rho}}{\rho}} \rho \,d\lambda \nonumber \\
    & = \int_{X \times Y}  \hat{\rho} \log \hat{\rho} \,d\lambda - \int_{X\times Y} \log\rho \,d \pidata \label{eq:KL-pidata-pic} \\
    & = H(\pidata) - \varepsilon^{-1}\int_{X \times Y} (\alpha^c+\beta^c-c)  \,d \pidata . \nonumber 
\end{align}
Since $\pidata \in \Pi(\mu,\nu)$, we know
\begin{equation}
    \label{eq:alpha-beta-red}
    \int_{X \times Y} \alpha^c \,d\pidata  = \int_X \alpha^c \,d\mu\quad \mbox{and} \quad
    \int_{X \times Y} \beta^c \,d\pidata = \int_Y \beta^c \,d\nu
\end{equation}
Plugging \eqref{eq:rho} and \eqref{eq:KL-pidata-pic} into \eqref{eq:irot_upper}, eliminating the constant $H(\pidata)$ which is independent of $c$, and multiplying the objective function by $\varepsilon>0$ which does not alter minimization, we obtain:
\begin{equation}
\label{eq:irot_equiv}
\min_{c}\cbr[3]{R(c) - \int_X \alpha^c \,d\mu - \int_Y \beta^c \,d\nu + \int_{X \times Y} c \,d\pidata }.
\end{equation}

In \eqref{eq:irot_equiv}, $\alpha^c$ and $\beta^c$ are the optimal solution of the dual problem of \eqref{eq:irot_lower} and hence implicitly depend on $c$. To make them independent variables, we plug $(\alpha^c, \beta^c)$ into $E$ defined in \eqref{eq:def_E} and obtain:
\begin{align}
E(\alpha^c,\beta^c,c)
& = \int_{X \times Y} c \,d\pidata - \int_X \alpha^c \,d\mu - \int_Y \beta^c \,d \nu + \varepsilon \int_{X \times Y} e^{(\alpha+\beta-c)/\varepsilon} \,d\lambda \nonumber \\
& = \int_{X \times Y} c \,d\pidata - \int_X \alpha^c \,d\mu - \int_Y \beta^c \,d \nu + \varepsilon, \label{eq:rot_dual_sub}
\end{align}
where the last equality is due to the unity property $\int_{X \times Y} d \pi^c = 1$ since $\pi^c \in \Pi(\mu,\nu)$ is a joint probability.
On the other hand, for any $c$, the optimality of $(\alpha^c, \beta^c)$ to the dual problem \ref{eq:rot_dual} implies that
\begin{equation}
\label{eq:rot_dual_oc}
    E(\alpha^c,\beta^c,c) = \min_{\alpha,\beta} E(\alpha,\beta,c).
\end{equation}
Combining \eqref{eq:rot_dual_sub} and \eqref{eq:rot_dual_oc} and plugging to \eqref{eq:irot_equiv}, merging the minimizations, and eliminating the singled-out constant $\varepsilon$, we obtain \eqref{eq:irot_final}.

To this point, we have showed that, for any fixed $c$, there is
\begin{equation*}
    E(\alpha^c,\beta^c,c)+R(c) - \varepsilon = \varepsilon (\KL(\pidata,\pi^c) + \varepsilon^{-1} R(c)),
\end{equation*}
where the left hand side is the objective function in \eqref{eq:irot_final} minus the constant $\varepsilon$, and the right hand side is $\varepsilon$ multiple of the objective function in \eqref{eq:irot_upper}. Therefore, the two minimization problems are equivalent and share the same set of solutions $c^*$.
\end{proof}

The variational model \eqref{eq:irot_final} is the foundation of our algorithmic development for cost function learning in the next section.
Compared to \eqref{eq:irot}, the optimization problem \eqref{eq:irot_final} consists of a convex functional $E$ (as shown later) and a customizable regularization (or constraint) $R$, and hence has the potential to be solved much more efficiently than posed as a bi-level optimization problem \eqref{eq:irot}.
Notice that, if $c$ is given and fixed, then the inverse OT \eqref{eq:irot_final} reduces to the dual problem of the (forward) entropy regularized OT \eqref{eq:rot_dual}.

As we will show below, the key feature of \eqref{eq:irot_final} is that the functional $E(\alpha,\beta,c)$ is \emph{jointly convex} in $(\alpha,\beta,c)$. That is, $E$ is a convex functional defined on 
\[
\Wcal := \mathcal{C}(X)\times \mathcal{C}(Y) \times \mathcal{C}(X\times Y),
\]
where $\mathcal{C}(X)$ stands for the set of all real-valued continuous functions on $X$.
However, unlike \eqref{eq:rot_dual}, $E(\alpha,\beta,c)$ is not strictly convex in its variable $(\alpha,\beta,c)$ and thus we cannot claim uniqueness of its minimizer. Indeed, we will show that there are infinitely many solutions to \eqref{eq:irot_equiv} with the same $\pidata$ when no additional regularization/constraint $R$ is imposed.
In order to characterize the solution set of \eqref{eq:irot_final}, we first need to investigate the behavior of $E$ over the quotient space induced by the following equivalence relation.
\begin{definition}
\label{def:equiv}
We say that $(\alpha,\beta,c)$ and $(\bar{\alpha},\bar{\beta},\bar{c})$ are \emph{equivalent}, denoted by $(\alpha,\beta,c) \sim (\bar{\alpha},\bar{\beta},\bar{c})$, if $\alpha(x)+\beta(y) - c(x,y) = \bar{\alpha}(x)+\bar{\beta}(y) - \bar{c}(x,y)$ for any $x \in X$ and $y\in Y$. 
\end{definition}
It is easy to verify that $\sim$ defines an \emph{equivalence relation} over $\Wcal$ in the classical sense. This equivalence relation induces a \emph{quotient space} $\Wcaltd := \Wcal/\sim$. Denote $[(\alpha,\beta,c)]\subset \Wcal$ the equivalence class of $(\alpha,\beta,c)$. Then $P: \Wcal \to \Wcaltd$ defined by $P(\alpha,\beta,c) = [(\alpha,\beta,c)]$ is called the \emph{canonical projection}.
We have the following result regarding the functional $E(\alpha,\beta,c)$ in \eqref{eq:irot_final}.
\begin{theorem}\label{thm:convex}
The following statements hold for the functional $E(\alpha,\beta,c)$ defined in \eqref{eq:irot_final}:
\begin{enumerate}
    \item[(i)] $E(\alpha,\beta,c)$ is jointly convex in $(\alpha,\beta,c)$;
    \item[(ii)] $E$ is a constant on each equivalence class $[(\alpha,\beta,c)]$;
    \item[(iii)] Let $\Etd: \Wcaltd \to \mathbb{R}$ be such that $\Etd([(\alpha,\beta,c)]) = E(\alpha,\beta,c)$, then $\Etd$ is well defined. 
    \item[(iv)] If $(\alpha^*,\beta^*,c^*)$ is a solution of \eqref{eq:irot_final}, then the corresponding optimal transport plan $\pi^{*}$ in \eqref{eq:irot} is given by $d \pi^{*} = e^{(\alpha^*(x) + \beta^* (x) - c^*(x,y))/\varepsilon} d\lambda$. Moreover, $[(\alpha^*,\beta^*,c^*)]$ is the unique minimizer of $\Etd$.
\end{enumerate}
\end{theorem}

\begin{proof}
(i) Let $\phi: X\times Y \to \mathbb{R}^3$ be $\phi(x,y) := (\alpha(x),\beta(y),c(x,y))$ for any $x \in X$ and $y \in Y$. Denote $\zeta=(1,1,-1)\in\mathbb{R}^3$. 
Then the functional $E$ in \eqref{eq:irot_final} is given by
\begin{align*}
    E(\phi) = \int_{X \times Y} \phi \cdot (-d\mu,-d\nu,d\pidata) + \varepsilon \int_{X \times Y} e^{(\zeta \cdot \phi)/\varepsilon} \, d\lambda.
\end{align*}
For any fixed $\psi: X \times Y \to \mathbb{R}^3$, we define the variation $f: I \to \mathbb{R}$, where $I \subset \mathbb{R}$ is a small open neighborhood of $0$, as follows,
\begin{equation*} 
f(\epsilon) := E(\phi + \epsilon \psi).
\end{equation*}
Then we can verify that
\begin{equation*}
f'(\epsilon) = \int_{X \times Y} \psi \cdot (-d\mu,-d\nu,d\pidata) + \int_{X \times Y} e^{(\zeta \cdot (\phi+\epsilon \psi))/\varepsilon} (\zeta\cdot \psi)\, d\lambda
\end{equation*}
and that
\begin{equation*}
f''(\epsilon) = \frac{1}{\varepsilon} \int_{X \times Y} e^{(\zeta \cdot (\phi+\epsilon \psi))/\varepsilon} (\zeta\cdot \psi)^2\, d\lambda \ge 0
\end{equation*}
for any $\epsilon \in I$. Since $\psi$ is arbitrary, we know $E$ is a convex functional of $\phi$.

(ii) To show that $E$ is constant over the equivalence class $[\phi]=[(\alpha,\beta,c)]$, we suppose $(\alpha,\beta,c)\sim (\bar{\alpha},\bar{\beta},\bar{c})$, i.e., $\alpha(x) + \beta(y) - c(x,y) = \bar{\alpha}(x) + \bar{\beta}(y) - \bar{c}(x,y)$ for all $(x,y)$. 
Denote $\delta_{\alpha}(x) := \alpha(x)-\bar{\alpha}(x)$ for every $x\in X$ and $\delta_{\beta}(y) := \beta(y) - \bar{\beta}(y)$ for every $y\in Y$, then it is obvious that
\begin{equation*}
\bar{c}(x,y) = \bar{\alpha}(x) + \bar{\beta}(y) - \alpha(x) - \beta(y) + c(x,y) = -\delta_{\alpha}(x) - \delta_{\beta}(y) + c(x,y).
\end{equation*}
Hence we have
\begin{align*}
E(\bar{\alpha},\bar{\beta},\bar{c})
=\ & \int_{X \times Y} \bar{c} \, d \pidata - \int_X  \bar{\alpha} \, d \mu - \int_Y \bar{\beta} \, d \nu + \varepsilon \int_{X \times Y} e^{(\bar{\alpha}+\bar{\beta}-\bar{c})/\varepsilon} \, d\lambda \\
=\ & \int_{X \times Y} (-\delta_{\alpha} - \delta_{\beta} + c) \, d \pidata - \int_X (\alpha-\delta_{\alpha}) \, d \mu - \int_Y (\beta-\delta_{\beta}) \, d \nu + \varepsilon \int_{X \times Y} e^{(\alpha+\beta-c)/\varepsilon} \, d\lambda \\
=\ &  E(\alpha,\beta,c) - \int_{X \times Y} (\delta_{\alpha}+\delta_{\beta}) \, d \pidata + \int_X \delta_{\alpha} \, d \mu + \int_Y  \delta_{\beta} \, d \nu \\
=\ &  E(\alpha,\beta,c),
\end{align*}
where the last equality is a result of cancellations due to 
\begin{align*}
\int_{X \times Y} \delta_{\alpha}\,d\pidata &= \int_X \delta_{\alpha} \del[2]{\int_Y \, d\pidata} = \int_X \delta_{\alpha}\, d \mu \\
\int_{X \times Y} \delta_{\beta}\,d\pidata &= \int_Y \delta_{\beta} \del[2]{\int_X \, d\pidata} = \int_Y \delta_{\beta}\, d \nu
\end{align*}
by Fubini theorem.
Therefore $E$ is constant throughout the equivalence class $[(\alpha,\beta,c)]$.

(iii) As an immediate consequence of (ii), $\Etd: \Wcaltd \to \mathbb{R}$ with $\Etd([\phi]) := E(\phi)$ for every $\phi \in \Wcal$ is well defined.

(iv) For any minimizer $(\alpha^*,\beta^*,c^*)$ of \eqref{eq:irot_final}, we know that $(\alpha^*,\beta^*)$ minimizes $E(\alpha,\beta,c)$ when $c=c^*$, i.e.,
\[
(\alpha^*,\beta^*) = \argmin_{\alpha,\beta} E(\alpha,\beta,c^*) = \argmin_{\alpha,\beta} \cbr[2]{\varepsilon \int_{X \times Y} e^{(\alpha+\beta-c^*)/\varepsilon}\, d\lambda - \int_X\alpha\, d \mu - \int_Y  \beta  \, d \nu}.
\]
Hence $(\alpha^*,\beta^*)$ is the optimal dual variable of the forward OT with cost $c^*$, and therefore the optimal transport plan (optimal primal variable) is $d \pi^*(x,y)=e^{(\alpha^*(x)+\beta^*(y)-c^*(x,y))/\varepsilon} d\lambda$ for all $(x,y)\in X \times Y$.

To show that $[\phi^*]$ is the unique minimizer of $\Etd$ over $\Wcaltd$ when $\phi^*$ minimizes $E$, we define for any nonzero $[\psi]\in \Wcaltd$ the variation $g: I \to \mathbb{R}$, where $I \subset \mathbb{R}$ is an open neighborhood of $0$, as follows,
\begin{equation*}
g(\epsilon) = \Etd([\phi] + \epsilon [\psi]) = \Etd([\phi  + \epsilon \psi]) = E(\phi + \epsilon \psi),
\end{equation*}
where we used the fact $[\phi]  + \epsilon [\psi] = [\phi  + \epsilon \psi]$ for any $\phi,\psi \in \Wcal$ and $\epsilon \in I$, which can be easily deduced from Definition \ref{def:equiv}, to obtain the second equality.
Following the same derivation as in (i), we can show that.
\[
g''(\epsilon) = \frac{1}{\varepsilon}\int_{X \times Y} e^{(\zeta \cdot (\phi + \epsilon \psi))/\varepsilon}(\zeta \cdot \psi)^2 \, d\lambda.
\]
Since $[\psi] \ne 0$, we know $\zeta \cdot \psi(x,y) \ne 0$ at some $(x,y)\in X \times Y$. Since $\psi$ is continuous, we know that there exists $\delta>0$ and an open neighborhood $U\subset X \times Y$ (with positive measure $\lambda(U)>0$) of $(x,y)$ such that $e^{(\zeta \cdot (\phi + \epsilon \psi))/\varepsilon}(\zeta \cdot \psi)^2 \ge \delta > 0$ for all $(x,y) \in U$. This implies that
\[
g''(\epsilon) \ge \frac{1}{\varepsilon } \int_{U}\delta \, d\lambda = \varepsilon^{-1}\delta \lambda(U) > 0.
\]
Hence $g$ is strictly convex at every $[\phi] \in \Wcaltd$. 
Therefore $[\phi^*]$ is the unique minimizer of $\Etd$ on the quotient space $\Wcaltd$.
\end{proof}

According to Theorem \ref{thm:convex}, our inverse OT formulation \eqref{eq:irot_final} is convex as long as the customizable regularization $R(c)$ is convex in $c$ or imposes a constraint of $c$ onto a convex set.
In this case, we can employ convex optimization schemes to solve \eqref{eq:irot_final} for the cost function, which are computationally much cheaper than solving general bi-level optimizations. 
Moreover, Theorem \ref{thm:convex} implies that \eqref{eq:irot_final} admits a unique equivalence set that minimizes the functional $E$. 
Therefore, even without $R(c)$, we can characterize the entire optimal solution set using one minimizer of $E(\alpha,\beta,c)$: if $(\alpha,\beta,c)$ minimizes $E$, then any other $(\bar{\alpha},\bar{\beta},\bar{c})$ minimizes $E$ \emph{if and only if} $(\bar{\alpha},\bar{\beta},\bar{c}) \sim (\alpha,\beta,c)$.
As we will show later, certain mild regularization $R(c)$ on $c$ can further narrow down the search of the desired cost $c$ to a single point within the equivalence set minimizing $E$.
%

\section{Algorithmic Development}
\label{sec:algorithm}
In this section, we develop prototype numerical algorithms for cost function learning based on the inverse OT formulation \eqref{eq:irot_final}. 
As mentioned in Section \ref{sec:intro}, the cost function reduces to a matrix in the discrete case while renders a real-valued function on $X \times Y$ in the continuous case. Due to this substantial difference, we consider the algorithmic developments in these two cases separately.
\subsection{Discrete Case}
\label{subsec:dis_iot}
\paragraph{Algorithm for discrete inverse OT}
In the discrete case, the marginal distributions $\mu$ and $\nu$ are two probability vectors from $\Delta^{m-1}$ and $\Delta^{n-1}$, respectively.
Therefore, the cost $c$ and transport plan $\pi$ are both $m\times n$ matrices.
Suppose that we can summarize the observed matching pairs into the matching matrix $\pidata$, e.g., $\pi_{ij}=N_{ij}/N$, where $N_{ij}$ is the number of couples of an individual from the $i$th class corresponding to $\mu_i$ and another individual from the $j$th class corresponding to $\nu_j$, and $N=\sum_{i=1}^m\sum_{j=1}^{n} N_{ij}$, then the inverse OT formulation \eqref{eq:irot_final} reduces to the following optimization problem:
\begin{align}
\min_{\alpha,\beta,c}
& \ \{ R(c) - \langle \alpha, \mu \rangle - \langle \beta, \nu \rangle + \langle c, \pidata \rangle  +  s(\alpha,\beta,c)\} \label{eq:dis_irot_final}
\end{align}
where $\alpha \in \mathbb{R}^m$, $\beta \in \mathbb{R}^n$, $c\in\mathbb{R}^{m\times n}$, and $s: \mathbb{R}^{m} \times \mathbb{R}^{n} \times \mathbb{R}^{m\times n} \to \mathbb{R}$ is defined by
\begin{equation}
    \label{eq:def_s}
s(\alpha,\beta,c):=\varepsilon\sum_{i=1}^m \sum_{j=1}^n e^{(\alpha_i + \beta_j - c_{ij})/\varepsilon}.
\end{equation}
To solve the minimization problem, we can apply a variant of matrix scaling by modifying the Sinkhorn-Knopp algorithm that alternately updates $\alpha,\beta,c$ in \eqref{eq:dis_irot_final}.
Specifically, the updates of $\alpha$ and $\beta$ are identical to that in the Sinkhorn-Knopp algorithm for the forward entropy regularized OT \citep{cuturi2013sinkhorn}.
The update of $c$ reduces to solving a regularized (or constrained, depending on $R(c)$) minimization for fixed $\alpha,\beta$. This update scheme is well known as block coordinate descent (BCD) or alternating minimization (AM). In general, convergence of BCD requires joint convexity objective function and Lipschitz continuity of its gradient \citep{beck2015convergence}. The objective function $E$ is shown to be joint convex above, but its gradient is not Lipschitz continuous. We present an alternate formulation which is equivalent to \eqref{eq:irot_final} but the objective function can be shown to have Lipschitz continuous gradient. We provide details of the BCD algorithm and its convergence for this formulation in Appendix \ref{appsec:bcd}.

We here advocate a modified algorithm which is easy-to-implement and empirically performs better than BCD in our tests. The updates of $\alpha$ and $\beta$ remain the same as before. The modification is in the update of $c$ as follows: we split the update of $c$ into two steps, where the first step is matrix scaling to compute an $m\times n$ matrix $K$ such that
\begin{equation*}
    K_{ij} = \frac{\pidata_{ij}}{e^{(\alpha_i + \beta_j)/\varepsilon}}, \quad \mbox{for}\quad i=1,\dots,m, \ \  j=1,\dots,n,
\end{equation*} 
and the second step is a proximal gradient descent to obtain the updated $c$:
\begin{equation}
\label{eq:prox-R}
    c = \prox_{\gamma R}(\hat{c}) := \argmin_{c} \cbr[2]{R(c) + \frac{1}{2\gamma} \|c - \hat{c}\|^2 },
\end{equation}
where $\hat{c} := -\varepsilon \log K$, and all exponential, logarithm, division operations mentioned here are performed component-wisely.
If $R(c)$ imposes a constraint of $c$ to a convex set $C \subset \mathbb{R}^{m\times n}$, then $\prox_{\gamma R}$ reduces to the orthogonal projection onto $C$, which has unique solution and usually can be computed very fast.
We summarize the steps of this modified scheme in Algorithm \ref{alg:dis_irot}.

It is worth noting that, as discussed in Section \ref{subsec:entropy_iot}, the value of $\varepsilon$ is \emph{unidentifiable} in inverse OT given that the data $\pidata$ only contains information $c/\varepsilon$ not $c$. Hence, we can just set $\varepsilon=1$ in Algorithm \ref{alg:dis_irot} and will recover $c/\varepsilon$, which is the same as $c$ up to a constant scaling. 

\begin{algorithm}[t]
\caption{Matrix Scaling Algorithm for Cost Learning in Discrete Inverse OT \eqref{eq:dis_irot_final}}
\label{alg:dis_irot}
\begin{algorithmic}
\STATE {\bfseries Input:} Observed matching matrix $\hat{\pi} \in \mathbb{R}^{m\times n}$ and its marginals $\mu \in\mathbb{R}^m,\nu\in\mathbb{R}^n$. 
\STATE {\bfseries Initialize:} $\alpha \in \mathbb{R}^{m\times 1}, \beta\in\mathbb{R}^{n\times 1}, u=\exp(\alpha/\varepsilon),v=\exp(\beta/\varepsilon)$, $c\in\mathbb{R}^{m\times n}$. 
\REPEAT
\STATE $K \gets e^{-c/\varepsilon}$
\STATE $u \gets \mu/(K v)$ 
\STATE $v \gets \nu/(K^{\top} u)$ 
\STATE $K \gets \hat{\pi}/(u v^{\top})$ 
\STATE $c \gets \prox_{\gamma R} (- \varepsilon \log(K))$ 
\UNTIL{convergent}
\STATE {\bfseries Output:} $\alpha = \varepsilon \log u$, $\beta= \varepsilon \log v$, $c$
\end{algorithmic}
\end{algorithm}


\paragraph{Uniqueness of solution in discrete inverse OT}
As inverse problems are underdetermined in general, additional information can be essential to narrow down the search to the desired solution to \eqref{eq:irot_final}.
The key is a properly designed $R(c)$ which imposes convex regularization or constraint to convex set.
To avoid overloading notations, we use the same symbols to denote the discrete counterparts of those in \ref{subsec:iot-formulation}.
For convenience, we denote $J=[1_m^{\top}\otimes I_n; I_m\otimes 1_n^{\top}; I_{mn}]$, where $I_n$ is the $n\times n$ identity matrix, and $[\cdot;\cdot]$ stacks the arguments vertically by following the standard MATLAB syntax.
Denote $\phi=[\alpha;\beta;c]\in\mathbb{R}^{m+n+mn}$ (where $c\in\mathbb{R}^{mn}$ stacks the columns of $c\in\mathbb{R}^{m\times n}$ in order vertically, we use the matrix and vector forms of $c$ interchangeably hereafter) , then $J\phi=0$ iff $\alpha_i + \beta_j = c_{ij}$ for all $i,j$. The following result characterizes a sufficient condition for unique minimizers of \eqref{eq:dis_irot_final} if $R(c)$ imposes a convex constraint of $c$ into the set $C$.
Note that, in this case, \eqref{eq:dis_irot_final} reduces to $\min_{(\alpha,\beta,c)\in \Wcal} E(\alpha,\beta,c)$ over the manifold $\Wcal=\mathbb{R}^{m} \times \mathbb{R}^n \times C$ describing the constraint on $\phi$, and the proximal operator $\prox_{\gamma R}$ in the $c$-step in Algorithm \ref{alg:dis_irot} is the projection onto $C$.
\begin{theorem}\label{thm:convex_dis}
Suppose $R(c)$ imposes the projection onto a closed convex set $C$ in $\mathbb{R}^{m\times n}$ and \eqref{eq:dis_irot_final} attains minimum at $\phi^*:=(\alpha^*,\beta^*,c^*)$. If $T_{\phi^*}\Wcal \cap \mathrm{ker}(J)=\{0\}$, where $T_{\phi^*}\Wcal$ is the tangent space of $\Wcal$ at $\phi^*$, then $\phi^*$ is the unique minimizer of \eqref{eq:dis_irot_final}.
\end{theorem}

\begin{proof}
In the discrete case, we use the notation $\phi=[\alpha;\beta;c]\in\mathbb{R}^{m+n+mn}$ and $E(\phi) = \langle [-\mu;-\nu;c], \phi\rangle + \varepsilon \langle e^{J\phi/\varepsilon}, 1_{mn}\rangle$, where the exponential is component-wise.
Hence the Hessian of $E(\phi)$ is
\begin{equation*}
    \nabla^2 E(\phi) = \varepsilon^{-1} J^{\top} \mathrm{diag}(e^{J\phi/\varepsilon})J \succeq 0.
\end{equation*}
If $\phi^*$ is a minimizer and $T_{\phi^*}\Wcal \cap \mathrm{ker}(J) = \{0\}$, then for any nonzero $\psi \in T_{\phi^*}\Wcal$, we have $J\psi \ne 0$. Hence,
\begin{equation*}
    \psi^{\top} (\nabla^2 E) \psi = \varepsilon^{-1} (J\psi)^{\top} \mathrm{diag}(e^{J\psi^*/\varepsilon}) (J\psi) > 0,
\end{equation*}
since $\mathrm{diag}(e^{J\psi^*/\varepsilon}) \succ 0$.
Therefore $\phi^*=(\alpha^*,\beta^*,c^*)$ is the unique global minimizer.
\end{proof}

We can derive closed-form expression for several special cases of $C$ that cover many of those used in practice.
\begin{example}[Symmetric cost matrix]
\label{ex:symmetric}
If $C$ is the set of symmetric matrices with zero diagonal entries, then $\prox_{\gamma R} (\hat{c})$ ($\gamma$ does not make any difference in this case) in Algorithm \ref{alg:dis_irot}, i.e., the projection of $\hat{c}$ onto $C$, is given by
\[
c = \prox_{\gamma R}(\hat{c}) = \frac{\hat{c} + \hat{c}^{\top}}{2}
\]
and followed by setting the diagonal entries $c$ to 0. Despite of its simple projection, the constraint $C$ includes a large number of cost matrices used in practice.
In particular, any (multiple of) distance-like cost matrix, i.e., $kc$ with $c_{ij}=|i-j|^p$ for any nonzero $k,p\in\mathbb{R}$ (if $p<0$ then we require $c_{ii}=0$ separately) is included in $C$. Note that any permutation of the indices of such $c$ is still in $C$.
\end{example}

Importantly, the following corollary shows that the symmetry of $c$ in Example \ref{ex:symmetric} implies uniqueness of solution to \eqref{eq:dis_irot_final}. Note that we additional assume the diagonal entries of $c$ to be zeros, but this holds in most applications since there is usually no cost incurred when no transfer of probability mass is needed.
\begin{corollary}\label{cor:sym}
Suppose $C=\{c \in\mathbb{R}^{n\times n}:\,c=c^{\top}, c_{ii}=0,\forall i\}$ and $R(\cdot)$ is the indicator function of $C$, i.e., $R(c)=0$ if $c \in C$ and $\infty$ otherwise, then \eqref{eq:dis_irot_final} has a unique solution $c^*$.
\end{corollary}
\begin{proof}
Since $R(\cdot)$ is the projection onto $C$, we know $\Wcal = \{\phi = (\alpha,\beta,c): L c = 0\}$ where $L \in \mathbb{R}^{(2n^{2}+n) \times n^{2}}$ represents the linear mapping such that $Lc \in \mathbb{R}^{2n^2+n}$ stacks vertically the $n\times n$ matrix $c$, its transpose $c^{\top}$ and the vector $\mathrm{diag}(c)$ of its diagonal. Therefore $T_{\phi}\Wcal = \Wcal$ for all $\phi \in \Wcal$. 

Suppose $\psi = (\alpha,\beta,c) \in T_{\phi^*}\Wcal \cap \mathrm{ker}(J)$, then $Lc=0$ and $c_{ij} = \alpha_i+ \beta_j$ for all $i,j$ (since $Jc = 0$). Therefore $c_{ii}=\alpha_{i} + \beta_{i}=0$, and hence $\alpha_i=-\beta_i$, for all $i$. Now we have
\begin{equation*}
c_{ij} = \alpha_i + \beta_j = \alpha_i - \alpha_j.
\end{equation*}
Similarly, $c_{ji} = \alpha_j - \alpha_i$. Since $c$ is symmetric, we know that $\alpha_i - \alpha_j = \alpha_j - \alpha_i$, which implies $\alpha_i=\alpha_j$. Therefore $c_{ij}=\alpha_i + \beta_j = \alpha_i-\alpha_j=0$ for all $i,j$, and hence $c=0$. As a consequence, $\alpha=-\beta=\xi 1_n$ for some constant $\xi\in\mathbb{R}$. So $T_{\phi}\Wcal \cap \mathrm{ker}(J)=\{(\xi 1_n,-\xi 1_n, 0)\in \mathbb{R}^{2n+n^2} :\, \xi\in\mathbb{R}\}$. 

By Theorem \ref{thm:convex}, if $\phi^*=(\alpha^*,\beta^*,c^*)$ solves \eqref{eq:dis_irot_final}, then $E$ attains minimum only at $\{\phi^*+\psi:\psi \in T_{\phi}\Wcal \cap \mathrm{ker}(J)\} \subset [\phi^*]$, which all contain the same cost matrix $c^*$.
\end{proof}

\begin{example}
[Linear affinity matrix] 
The cost matrix $c$ is parameterized as $c = G^{\top} A D$, where $G=[g_1,\dots,g_m]\in\mathbb{R}^{p \times m}$ and $D=[d_1,\dots,d_n] \in \mathbb{R}^{q \times n}$ are given. Here $G$ and $D$ stand for the feature vector matrices for the two populations in matching, and $A \in \mathbb{R}^{p \times q}$ is the so-called \emph{linear affinity matrix} (or interaction matrix) to be reconstructed. Therefore, the constraint set is $C=\{G^{\top} A D: A \in \mathbb{R}^{p\times q} \}$.

In the study of personality traits in marriage \citep{dupuy2014personality}, $g_i \in \mathbb{R}^p$ and $d_j \in \mathbb{R}^q$ are the given feature vectors of the $i$th class of men and $j$th class of women in the market for $i\in[m]$ and $j\in[n]$, and $A_{kl} \in \mathbb{R}$ is the complementary coefficient of the $k$th feature of men and $l$th feature of women for $k\in[p]$ and $l\in[q]$. 

In this case, the projection of $-\varepsilon \log K$ onto the set $C$ in \eqref{eq:prox-R} is given by 
\begin{equation*}
\prox_{R}(-\varepsilon (G^{+})^{\top} (\log K) D^{+}),
\end{equation*}
where $G^{+}$ and $D^{+}$ are the Moore-Penrose pseudoinverse of $G$ and $D$, respectively,and can be pre-computed before applying Algorithm \ref{alg:dis_irot}.
\end{example}

There are many other choices of the constraint set $C$ and the regularization $R$. These are often application-specific and thus require discussions case by case, which is beyond the scope of the present work. However, the general strategy developed in this section can be easily modified and applied to many different situations.

\subsection{Continuous Case}
\label{subsec:cts_iot}

In the continuous case, $\mu$, $\nu$, and $\pi$ represent probability density functions on $X$, $Y$, and $X \times Y$, respectively.
In this case, we are only given i.i.d.~samples of these distributions, namely, $x^{(i)} \sim \mu$, $y^{(j)}\sim \nu$, and $(x^{(i)},y^{(i)}) \sim \pidata \in \Pi(\mu,\nu)$.
The cost function $c(x,y)$ is a function defined on the continuous space $X \times Y$ where $X \subset \mathbb{R}^{d_X}$ and $Y \subset \mathbb{R}^{d_Y}$ for potentially high dimensionality $d_X$ and $d_Y$ (e.g., $d_{X},d_{Y} \ge 3$).
As mentioned in Section \ref{sec:intro}, discretization of the spaces $X$ and $Y$ renders $m$ and $n$ increasing exponentially fast in $d_{X}$ and $d_{Y}$, and then Algorithm \ref{alg:dis_irot} (or any discrete inverse OT algorithm) becomes infeasible computationally.

\paragraph{Cost function parametrization by deep neural networks}
Our approach \eqref{eq:irot_final} is an unconstrained optimization with customizable regularization $R(c)$. This allows for a natural solution to overcome the issue of discretization using deep neural networks, which is a significant advantage over bi-level optimization formulations.
More precisely, we can parameterize the cost function $c$, as well as the functions $\alpha$ and $\beta$, in \eqref{eq:irot_final} as deep neural networks.
In this case, $\alpha$, $\beta$, and $c$ are neural networks with output layer dimension $1$ and input layer dimensions $d_X$, $d_Y$, and $d_X+d_Y$, respectively.
In particular, the design of the network architecture of $c$ may take the regularization $R(c)$ into consideration.
For example, if $R(c)$ suggests that $c\ge 0$, then we can set the activation function in the output layer as the rectified linear unit (ReLU) $\sigma(x):=\max(x,0)$.
We can also introduce an encoder $h_{\eta}$ to be learned, such that the cost $c_\eta$ is the standard Euclidean distance between encoded features, e.g., $c_\eta(x,y) = |h_\eta(x) - h_\eta(y)|$.
Nevertheless, in general, the architectures of the $\alpha$, $\beta$, and $c$ are rather flexible, and can be customized adaptively according to specific applications.
We will present several numerical results with architecture specifications used in our experiments.

To formalize our deep neural net approach for solving continuous inverse OT, we let $\theta$ denote the parameters of $\alpha$ and $\beta$ (in actual implementations, $\alpha$ and $\beta$ are two separate neural networks with different parameters $\theta_a$ and $\theta_b$ respectively, but we use $\theta$ for both to avoid overloaded notations). %
In addition, we use $\eta$ to denote the parameters of $c$.
Then we can solve for the optimal $(\theta^*, \eta^*)$ by minimizing the loss function $L$ of the network parameters $(\theta,\eta)$ based on \eqref{eq:irot_final} as follows:
\begin{align}
\min_{\theta, \eta} L(\theta,\eta) := R(c_\eta) - \Ehat_{\mu}[\alpha_\theta] - \Ehat_{\nu}[\beta_\theta] + \Ehat_{\pidata}[c_\eta] \label{eq:loss}  
 + \varepsilon \int_{X \times Y} e^{(\alpha_\theta(x) + \beta_\theta(y) - c_\eta(x,y))/\varepsilon} \, d\lambda.
\end{align}
In \eqref{eq:loss}, the empirical expectations are defined by the sample averages:
\begin{align*}
\Ehat_{\mu}[\alpha_\theta] := \frac{1}{N_\mu} \sum_{i=1}^{N_\mu} \alpha_\theta(x^{(i)}),\quad
\Ehat_{\mu}[\beta_\theta] := \frac{1}{N_\nu} \sum_{i=1}^{N_\nu} \beta_\theta(y^{(i)}),\quad
\Ehat_{\pidata}[c_\eta] := \frac{1}{N_{\pidata}} \sum_{i=1}^{N_{\pidata}} c_\eta(x^{(i)}, y^{(i)}),
\end{align*}
where $\Dcal_\mu:=\{x^{(i)} : i\in[N_\mu]\}$ and $\Dcal_\nu:=\{y^{(i)} : i\in[N_\nu]\}$ are i.i.d.~samples drawn from $\mu$ and $\nu$ respectively, and $\Dcal_{\pidata}:=\{(x^{(i)},y^{(i)}) : i\in[N_{\pidata}]\}$ are observed pairings of $\pidata$.
If only $\pidata$ is available, we can also substitute the samples for $\mu$ and $\nu$ by the first and second coordinates of the samples in $\{(x^{(i)},y^{(i)}) : i\in[N_{\pidata}]\}$.
The last integral in \eqref{eq:loss} can be approximated by numerical integration methods, such as Gauss quadrature and sample-based integrations.
For example, if $X\times Y$ is bounded, we can sample $N_s$ collocation points $\{(x^{(i)}, y^{(i)})  :  i \in [N_s]\}$ from $X \times Y$ uniformly, and approximate the integral by
\begin{align}
\label{eq:int_unif}
& \int_{X \times Y} G_{\theta,\eta}(x,y) \, d \pidata
\approx \frac{1}{N_s} \sum_{i=1}^{N_s} G_{\theta,\eta}(x^{(i)}, y^{(i)}),
\end{align}
where $G_{\theta,\eta}(x,y) := e^{(\alpha_\theta(x) + \beta_\theta(y) - c_\eta(x,y))/\varepsilon}$.
A more appealing method for sample-based integration is to use an importance sampling strategy: we first estimate the mode(s) of the function $G_{\theta,\eta}(x,y)$, and draw i.i.d.~samples points $\{(x^{(i)}, y^{(i)})  :  i \in [N_s]\}$ from a Gaussian distribution $\rho((x,y); \omega,\Sigma)$ where $\omega$ and $\Sigma$ represent the mean (close to the mode) and variance of the Gaussian (or a mixture of Gaussians), and approximate the integral by \[
\int_{X \times Y} G_{\theta,\eta}(x,y) \, d \pidata
\approx \frac{1}{N_s} \sum_{i=1}^{N_s} \frac{G_{\theta,\eta}(x^{(i)}, y^{(i)})}{\rho((x^{(i)}, y^{(i)}); \omega,\Sigma)}.\]
The advantages of this importance sampling strategy include the capability of integral over unbounded domain $X \times Y$ and smaller sample approximation variance with properly chosen $\omega$ and $\Sigma$.
Other methods for approximating the integrals can also be applied. 
In our experiments, we simply used the uniform sampling shown in \eqref{eq:int_unif}.

Now we have all the ingredients in the loss function $L(\theta,\eta)$ in \eqref{eq:loss}.
We can apply (stochastic) gradient descent algorithm to $L$ and find an optimal solution $(\theta^*, \eta^*)$.
In each iteration, we can use all the samples available for the empirical expectations, or only sample a mini-batch for the computation of the gradient of $L$ with respect to $(\theta,\eta)$.
Otherwise, the optimization is standard in deep neural network training.
Moreover, we can use scaling $\alpha_\theta \leftarrow \alpha_\theta / \varepsilon$, $\beta_\theta \leftarrow \beta_\theta / \varepsilon$, and $c_\eta \leftarrow c_\eta / \varepsilon$ and hence $L(\theta,\eta)$ can be minimized with $\varepsilon=1$ in \eqref{eq:loss}.
Then we can scale $c_\eta$ back by multiplying $\varepsilon$ after $(\theta^*,\eta^*)$ is obtained.
This algorithm is summarized in Algorithm \ref{alg:cts_irot}.
\begin{algorithm}[t]
\caption{Cost Function Learning for Continuous Inverse OT by minimizing \eqref{eq:loss}}
\label{alg:cts_irot}
\begin{algorithmic}
\STATE {\bfseries Input:} Marginal distributions $\mu,\nu$ and observed pairing data $\hat{\pi}$. 
\STATE {\bfseries Initialize:} Deep nets $(\alpha_\theta,\beta_\theta)$, $c_\eta$.
\REPEAT
\STATE 1.~Draw a mini-batch from $\Dcal_{\mu}, \Dcal_{\nu}, \Dcal_{\pidata}$.
\STATE 2.~Sample $\{(x^{(i)}, y^{(i)})  :  i \in [N_s]\}\subset X \times Y$.
\STATE 3.~Form stochastic gradient $\hat{\nabla} L$ with empirical expectations and integral \eqref{eq:int_unif}.
\STATE 4.~Update $(\theta,\eta) \leftarrow (\theta,\eta) - \tau \hat{\nabla} L(\theta,\eta)$.
\UNTIL{convergent}
\STATE {\bfseries Output:} $\alpha_\theta$, $\beta_\theta$, $c_\eta$.
\end{algorithmic}
\end{algorithm}

\section{Numerical Experiments}
\label{sec:experiment}

\paragraph{Experiment setup}
We evaluate the proposed cost learning algorithms (Algorithms \ref{alg:dis_irot} and \ref{alg:cts_irot}) using several synthetic and real data sets.
Both algorithms are implemented in Python, where PyTorch is used in Algorithm \ref{alg:cts_irot} in the continuous inverse OT problem. The experiments are conducted on a machine equipped with 2.80GHz CPU, 16GB of memory.
To evaluate the cost matrices/functions $c$ learned by the algorithms when the ground truth cost $c^*$ is available, we use the criterion of relative error $\|c - c^*\| / \|c^*\|$.
For discrete case, $\|\cdot\|$ is the standard Frobenius norm of matrices.
For continuous case, we evaluate the learned function $c_{\theta}$ and the ground truth cost function $c^*$ at a given finite set of grid points in $X \times Y$, so that both $c$ and $c^*$ can be treated as vectors and the standard 2-norm can be applied.
%

\subsection{Discrete Inverse OT on Synthetic Data}
\label{subsec:dis_iot_synthetic}
We first test Algorithm \ref{alg:dis_irot} on learning cost matrix $c$ using observed transport plan matrix $\pidata$ in the discrete case.
We set $m=n$ and set ground truth $c^*$ with $c^*_{ij} = |\frac{i-j}{n}|^p$ for $i,j\in[n]$ and $p=0.5,1,2,3$.
Then we generate $\pidata$ for each $c^*$ with varying $\varepsilon=10^{1},10^{0},10^{-1},10^{-2}$ using Sinkhorn algorithm \citep{peyre2019computational}, and apply Algorithm \ref{alg:dis_irot} to $\pidata$ and see if we can recover the original $c^*$.
To this end, we set the constraint set $C=\{c\in\mathbb{R}^{n\times n}:\, c=c^{\top}, c_{ii}=0,\forall i\in[n]\}$.
%
We also truncate $c$ to be nonnegative values by applying $\max(\cdot,0)$, which seems to further improve efficiency for this problem.

Figure \ref{fig:dis_synthetic} shows the results of Algorithm \ref{alg:dis_irot}.
For fixed $\varepsilon=10^{-1}$, we generate $20$ random pairs of $(\mu,\nu)\in\mathbb{R}^{m} \times \mathbb{R}^{n}$ and corresponding $\pidata$ for problem size $m=n=100$, and apply Algorithm \ref{alg:dis_irot} to recover the cost matrix $c$. Figures \ref{subfig:dis_p_obj_vs_iter} and \ref{subfig:dis_p_err_vs_iter} show the average over the 20 instances of the objective function \eqref{eq:dis_irot_final} and relative error of $c$ (in logarithm) versus iteration number.
In all cases, the true $c^*$ is accurately recovered with relative error approximately $10^{-4}$ or lower after 500 iterations.
%
%
For fixed $p=2$, we also perform the same test of Algorithm \ref{alg:dis_irot} on $\pidata$ generated using varying entropy regularization weight $\varepsilon=10^{1},10^{0},10^{-1},10^{-2}$. We again run 20 instances and plot the relative error (in logarithm) versus iteration. The result is shown in Figure \ref{subfig:dis_eps_err_vs_iter}. With the same settings for $p=2$ and varying $\varepsilon$, we test Algorithm \ref{alg:dis_irot} for each problem size $n=128,256,512,1024,2048$, run the algorithm until the relative error of $c$ reaches $5 \times 10^{-2}$, and record the average of the CPU time for 20 instances. We plot the CPU time (in seconds) versus the problem size $n$ (in log-log) in Figure \ref{subfig:dis_eps_time_vs_size}.
We can see the algorithm run with smaller $\varepsilon$ reaches the prescribed relative error in shorter time as Figure \ref{subfig:dis_eps_time_vs_size} shows. From Figure \ref{subfig:dis_eps_err_vs_iter} we see how relative errors decrease as the iteration increases with different $\varepsilon$, all errors converge in similar patterns.
These tests evidently show the high efficiency and accuracy of Algorithm \ref{alg:dis_irot} in recovering cost matrices in discrete inverse OT.
\begin{figure}
\centering
\begin{subfigure}{.24\textwidth}
\includegraphics[width=\linewidth]{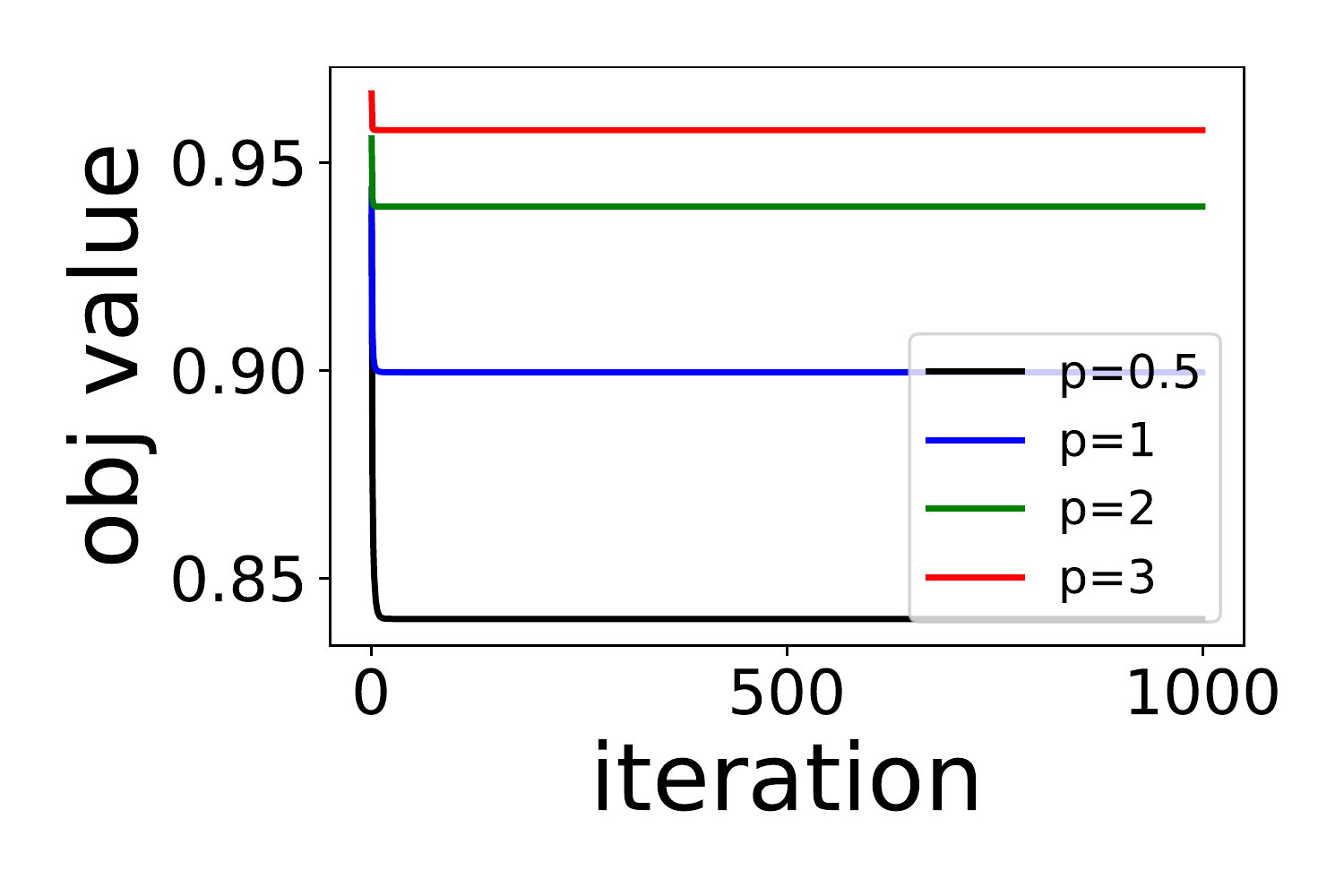}
\caption{}
\label{subfig:dis_p_obj_vs_iter}
\end{subfigure}
\begin{subfigure}{.24\textwidth}
\includegraphics[width=\linewidth]{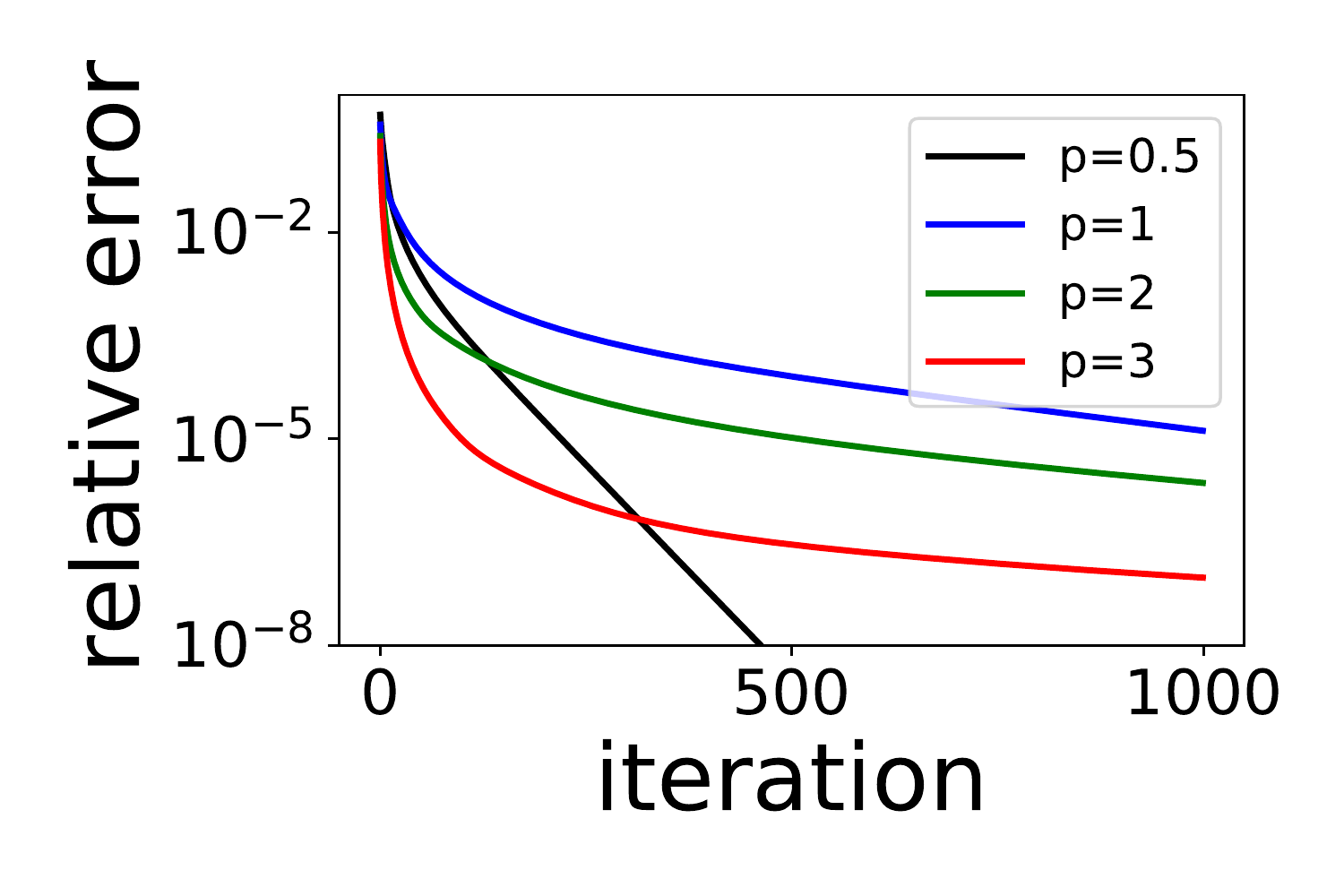}
\caption{}
\label{subfig:dis_p_err_vs_iter}
\end{subfigure}
\begin{subfigure}{.24\textwidth}
\includegraphics[width=\linewidth]{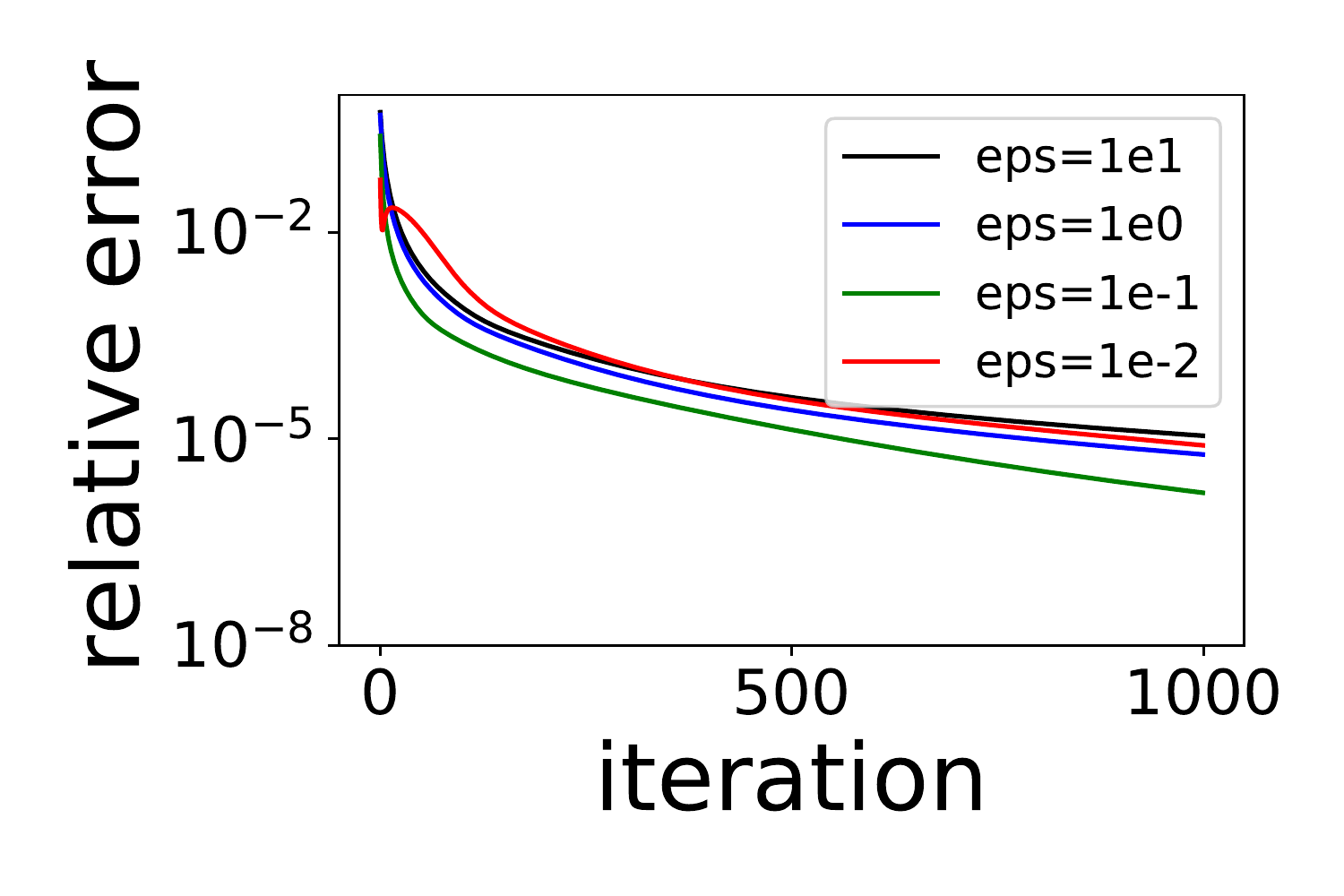} 
\caption{}
\label{subfig:dis_eps_err_vs_iter}
\end{subfigure}
\begin{subfigure}{.24\textwidth}
\includegraphics[width=\linewidth]{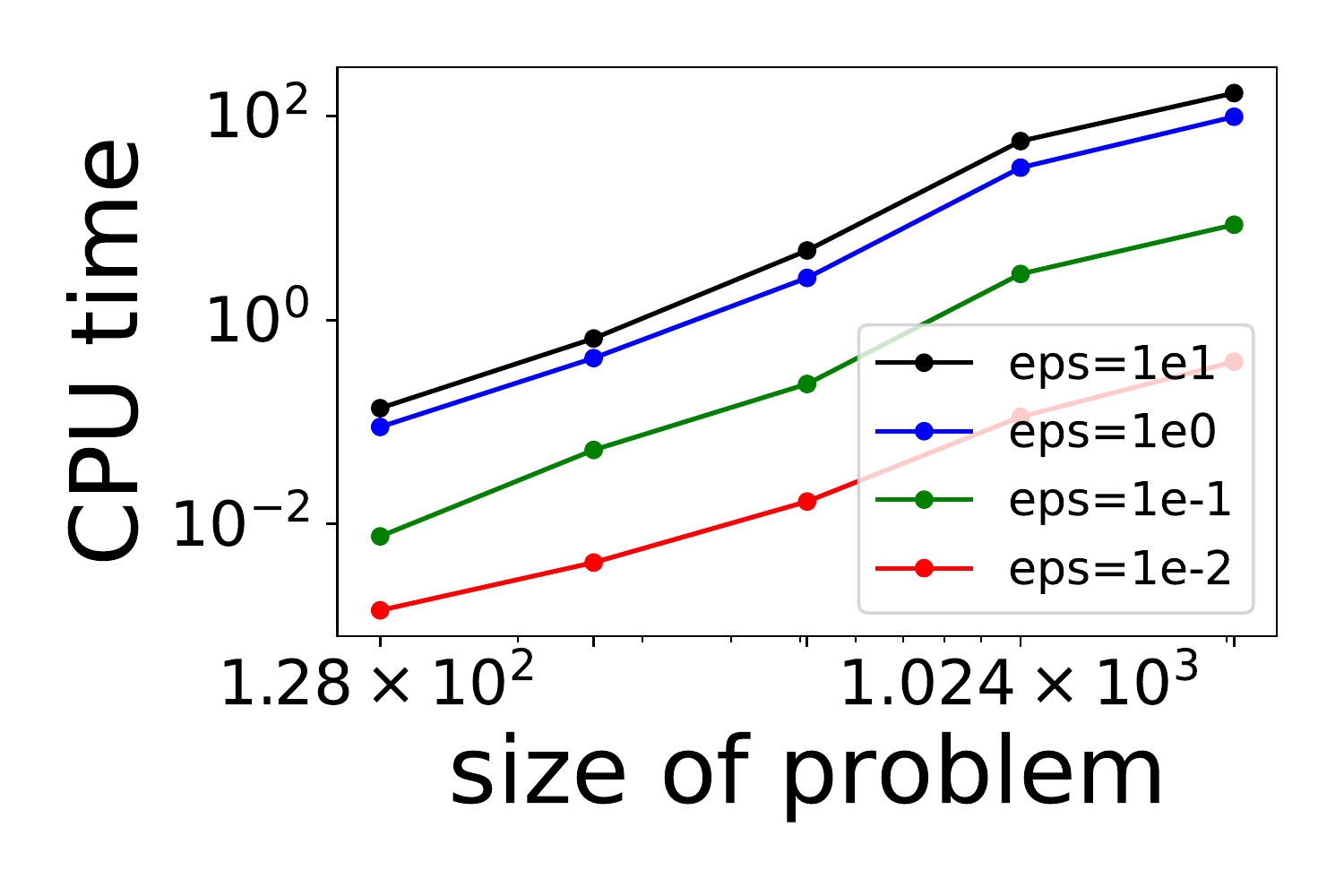} 
\caption{}
\label{subfig:dis_eps_time_vs_size}
\end{subfigure}
\caption{Results of Algorithm \ref{alg:dis_irot} for cost matrix recovery on synthetic data. True $c_{ij}=|\frac{i-j}{n}|^p$ for $i,j\in[n]$. (a) Objective function value versus iteration number for varying $p$. (b) Relative error (in log scale) versus iteration number for varying $p$. (c) Relative error (in log scale) versus iteration for varying $\varepsilon$. (d) CPU time (in seconds) versus problem size $n$ in log-log scale for varying $\varepsilon$. Each curve shows the average over 20 instances.}
\label{fig:dis_synthetic}
\end{figure}

\subsection{Discrete Inverse OT on Real Marriage Data}
\label{subsec:dis_iot_marriage}
We follow the setting of \citep{galichon2015cupids} and apply Algorithm 1 to the Dutch Household Survey (DHS) data set (\url{https://www.dhsdata.nl}) to estimate the affinity matrix $A$. Here the cost $c$  has a parametric form $C=-G^{\top}A D$ where $G$ and $D$ are given feature matrices as described in Section \ref{subsec:dis_iot}.
Following \citep{galichon2015cupids}, we classify men and women into $m$ and $n$ categories respectively based on the given $G \in \mathbb{R}^{p\times m}$ and $D \in \mathbb{R}^{d\times n}$. These two matrices represent the corresponding feature vectors for men and women. The matrix $A\in\mathbb{R}^{p\times q}$ is the reward (affinity) matrix to be estimated, where the $(i,j)$ entry $A_{ij}$ measures the complementarity or substitutability between the $i$th attribute of men and the $j$th attribute of women. 


Considering the data consistency, we only use the data from 2004 to 2017 (2016 is excluded due to an incompleteness issue of the data). We select 9 features from the data set, including educational-level, height, weight, health and 5 personality traits which can be briefly summarized as irresponsible, disciplined, ordered, clumsy, and detail-oriented. All the features are rescaled onto $[0,1]$ interval. The men and women are both clustered into 5 types by applying k-means algorithm, and each type of men or women is represented by the corresponding cluster center. After data cleaning, the data set contains the information about these features collected from 4,553 couples. In our experiment, we set $\varepsilon=10^{-2}$. Since the initialization of K-means algorithm still affects the values of the estimates, we run the experiments 100 times with different fixed random seeds and take the average of the resulting affinity matrix as the final estimation. The estimated affinity matrix is given in Table \ref{tab:affinity}.
%

\begin{table}[t]
    \centering
    \caption{Affinity matrix estimated using Algorithm \ref{alg:dis_irot} on the marriage data. ``H'' and ``W'' stand for Husband and Wife respectively, ``Edu'' stands for Education, ``Irres'' stands for Irresponsibility, and ``Disc'' stands for Disciplined.}
    \label{tab:affinity}
    \begin{tabular}{crrrrrrrrr}
    \toprule
    {H\textbackslash}W & {Edu} & {Height} & {Weight} & {Health}  & {Irres} & {Disc} & {Order} & {Clumsy} & {Detail} \\ 
    \midrule
    Edu & 0.065 & -0.083 & -0.052 & -0.048 & 0.015 & -0.013 & -0.043 & -0.063 & -0.040  \\
    Height    & -0.056 & -0.461 & -0.280 & -0.239 &-0.054 & 0.182 &-0.232 &-0.338 &-0.247\\
    Weight    &-0.037 &-0.301 &-0.182 &-0.156 &-0.037 &0.122 &-0.151 &-0.219 &-0.161\\
    Health    & -0.035 &-0.018 &-0.009 &-0.014 &0.050&-0.125 &-0.006 &-0.033 &-0.009 \\
    Irres     &-0.017 & -0.371 &-0.226 &-0.188 & -0.055  & 0.215 & -0.194 & -0.253 & -0.202\\ 
    Disc      &-0.002 & 0.097 & 0.059 & 0.055 & 0.022 & 0.002 & 0.050 & 0.079 & 0.052 \\
    Order      & -0.057 & -0.309 & -0.187 & -0.162 & -0.034 & 0.097 & -0.150 & -0.235 & -0.163      \\
    Clumsy   & -0.020 & -0.143 & -0.086 & -0.075 & 0.013 & 0.008 & -0.075 & -0.107 & -0.079\\
    Detail  &-0.049 & -0.407 & -0.247 & -0.210  & -0.070 & 0.204 & -0.202 & -0.295 & -0.216 \\
    \bottomrule
    \end{tabular}
\end{table}

The affinity matrix reveals several important implicit phenomena about marriage market. The education factor gives the most significant complementarity among all the other features. The trade-off between different features is revealed by the off-diagonal coefficients which are significantly different from zero. Since men and women have different preferences for these attributes, the affinity matrix is not symmetric.


We also compare the performance of Algorithm \ref{alg:dis_irot} with other pair matching algorithms, including the state-of-art RIOT model \citep{li2019learning}, SVD model \citep{koren2009matrix}, item-based collaborative filtering model (itemKNN) \citep{cremonesi2010performance}, probabilistic matrix factorization model (PMF) \citep{mnih2008probabilistic}, and factorization machine model (FM) \citep{rendle2012factorization}. These models have been evaluated on DHS data set in \citep{li2019learning} so we just follow the same experimental protocol with 5-fold cross-validation. Note that in \citep{liu2019learning} the authors come up with a neural network model and achieves the same performance as RIOT on DHS dataset so we just omit the repetitive evaluation here. We train all models on training data set and measure the errors on validation data set by computing the root mean square error (RMSE) and the mean absolute error (MAE). We also run the experiment for 10 times and record the running time for all models. The results are given in Table \ref{tab:marriage-comparison}. As Table \ref{tab:marriage-comparison} shows, our method (Algorithm \ref{alg:dis_irot}) significantly outperforms all these existing methods in both accuracy and efficiency. Specifically, the RMSE of Algorithm \ref{alg:dis_irot} is \textbf{$2.46\times 10^{-11}$} and MAE is \textbf{$1.90 \times 10^{-11}$}, so they are rounded as $0.0$ in Table \ref{tab:marriage-comparison}. 
%


\begin{table}[t]
\centering
\caption{Average error of 5-fold cross-validation in RMSE and MAE ($\times 10^{-4}$) and average running time (in seconds) for all compared matching algorithms.}
\label{tab:marriage-comparison}
\begin{tabular}{crrr}
\toprule
\textbf{Method}  & \textbf{RMSE}	& \textbf{MAE}	& \textbf{Runtime} \\
\midrule
Random & 45.1	& 31.4	& 1.24 \\
PMF \citep{mnih2008probabilistic} & 114.5	& 64.9	& 0.67 \\
SVD \citep{koren2009matrix} & 109.8	& 62.4	& 0.73 \\
itemKNN \citep{cremonesi2010performance} & 1.8	& 1.3	& 0.97 \\
FM \citep{rendle2012factorization} & 7.4	& 5.6	& 48.51 \\
RIOT \citep{li2019learning} & 1.8	& 1.3	& 7.32 \\
\textbf{Algorithm~\ref{alg:dis_irot}} & \textbf{0.0}	& \textbf{0.0} 	& \textbf{0.04} \\
\bottomrule
\end{tabular}
\end{table}

\subsection{Continuous Inverse OT on Synthetic Data}
\begin{figure}[t!]
\centering
{\scriptsize 
\begin{tabular}{cc}
\includegraphics[width=.36\textwidth]{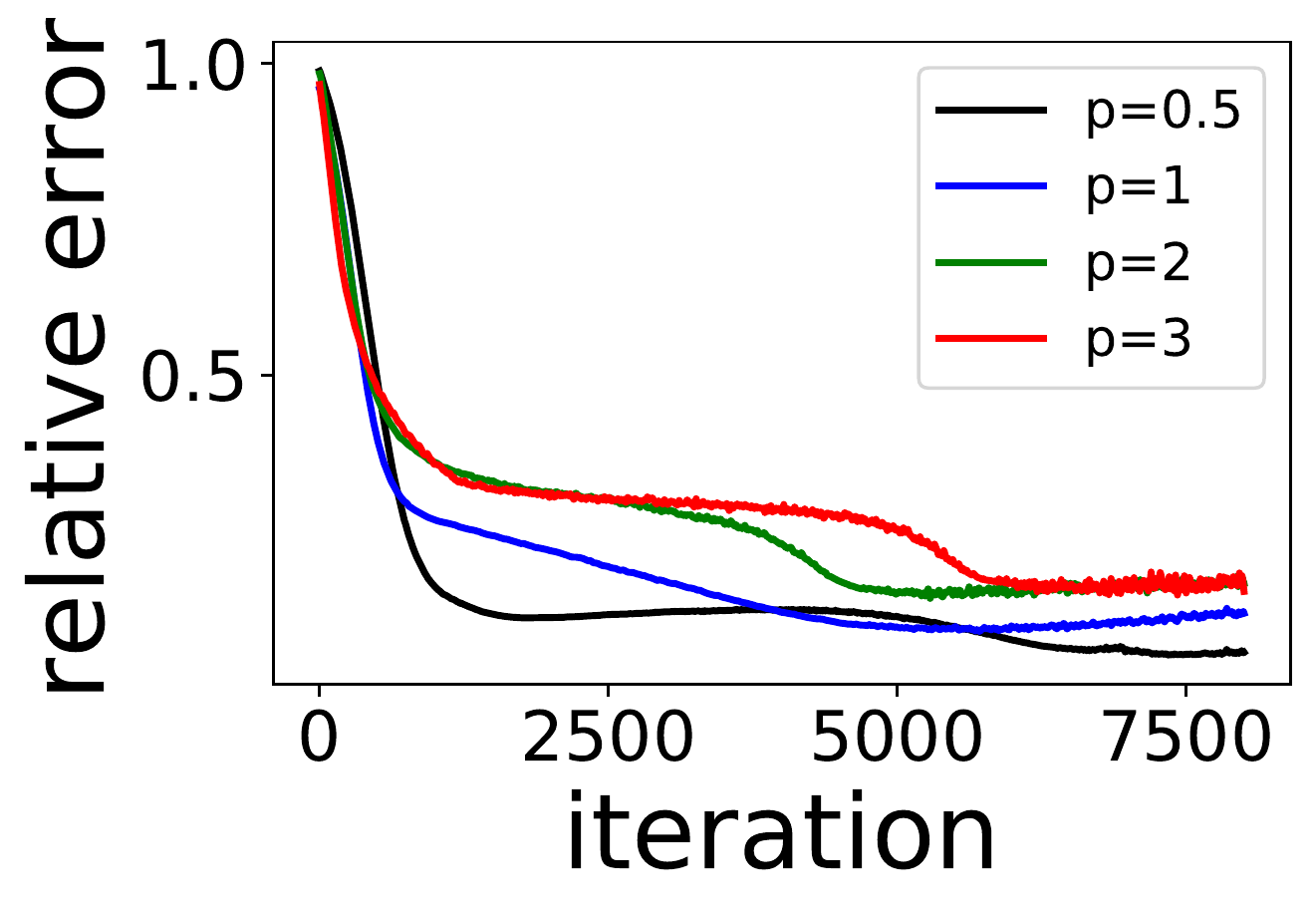} &
\includegraphics[width=.36\textwidth]{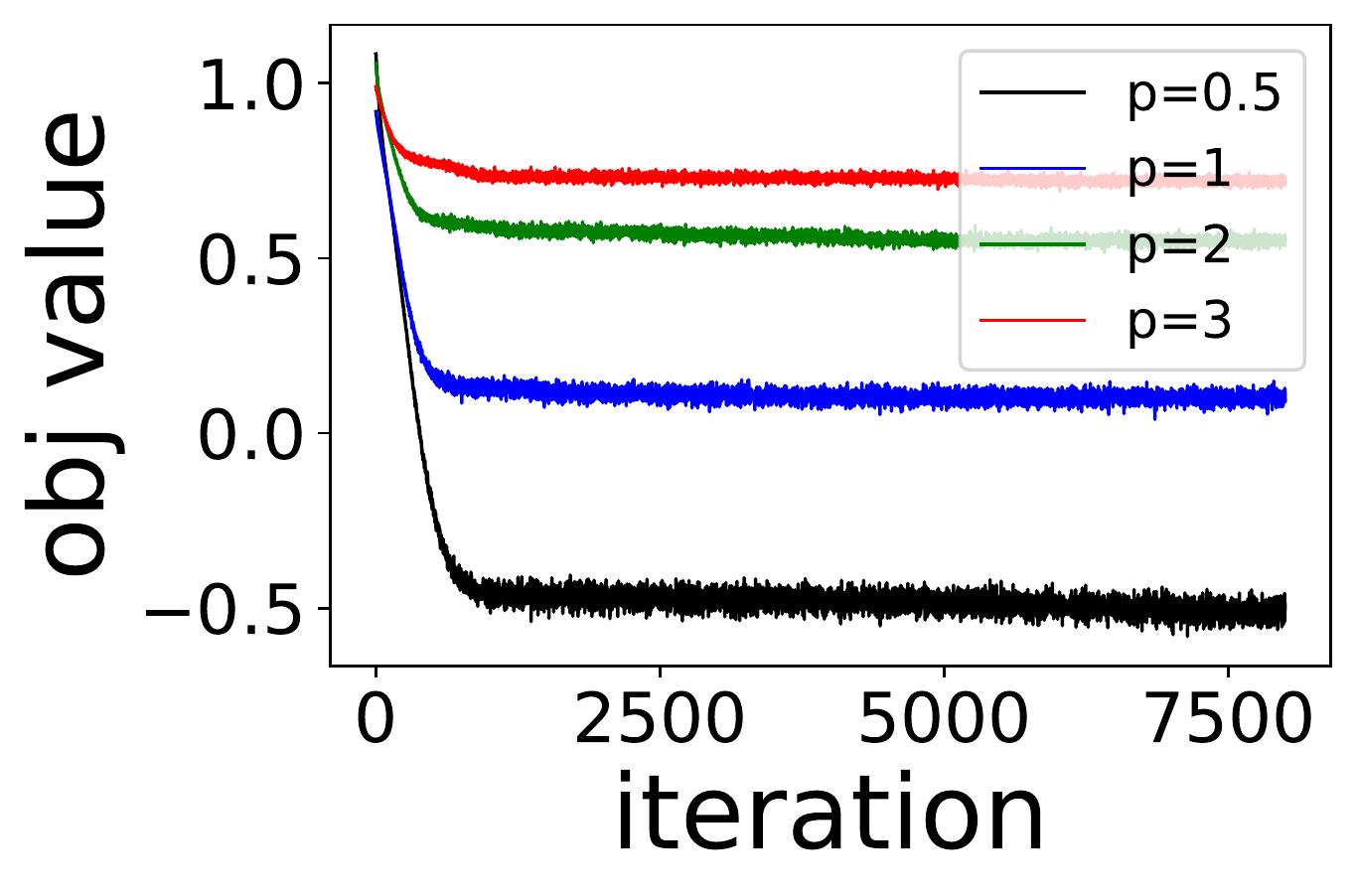} \\
\end{tabular}
}
\caption{The relative error (left) and objective function value (right) versus iteration number by  Algorithm \ref{alg:cts_irot} on the continuous inverse OT with synthetic data and varying $p$ in the symmetric case.}
\label{fig:cts_c_plot}
\end{figure}

\begin{figure}[t!]
\centering
\begin{subfigure}[b]{.48\columnwidth}
\includegraphics[width=\linewidth]{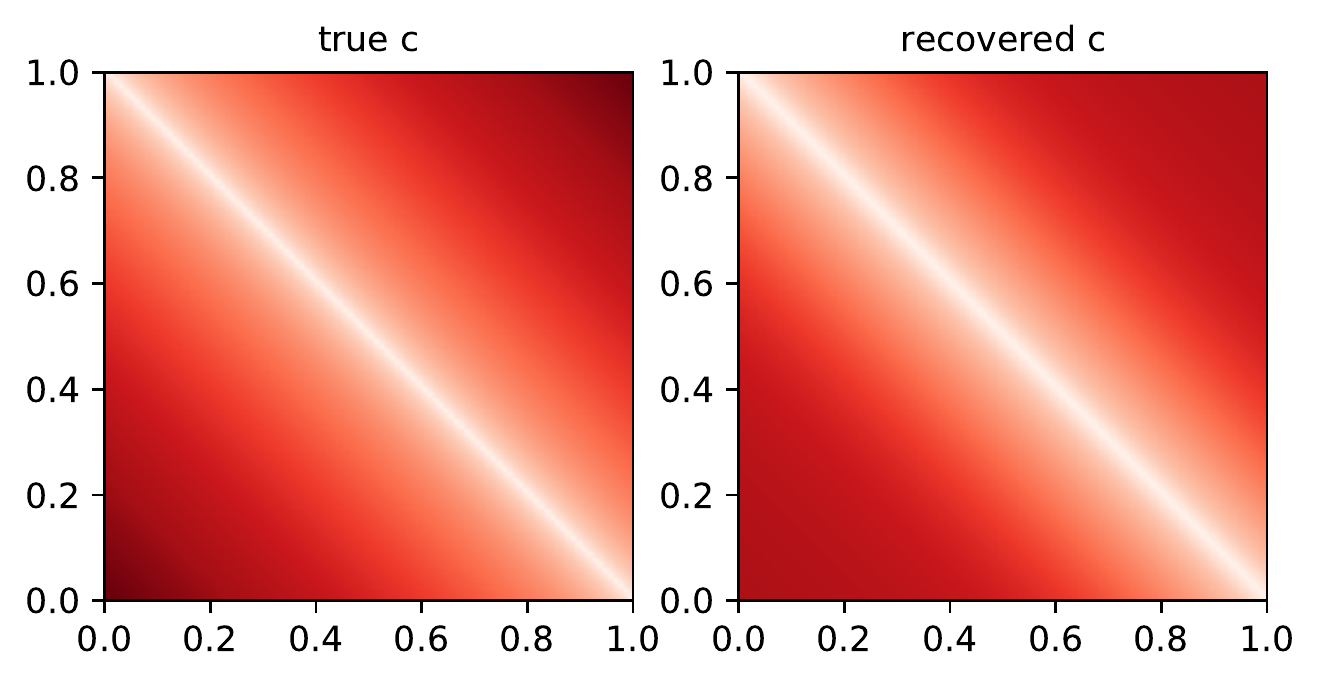} 
\caption{$|x-y|^{1/2}$}
\label{subfig:cts_c_rec_L0}
\end{subfigure}
\begin{subfigure}[b]{.48\columnwidth}
\includegraphics[width=\linewidth]{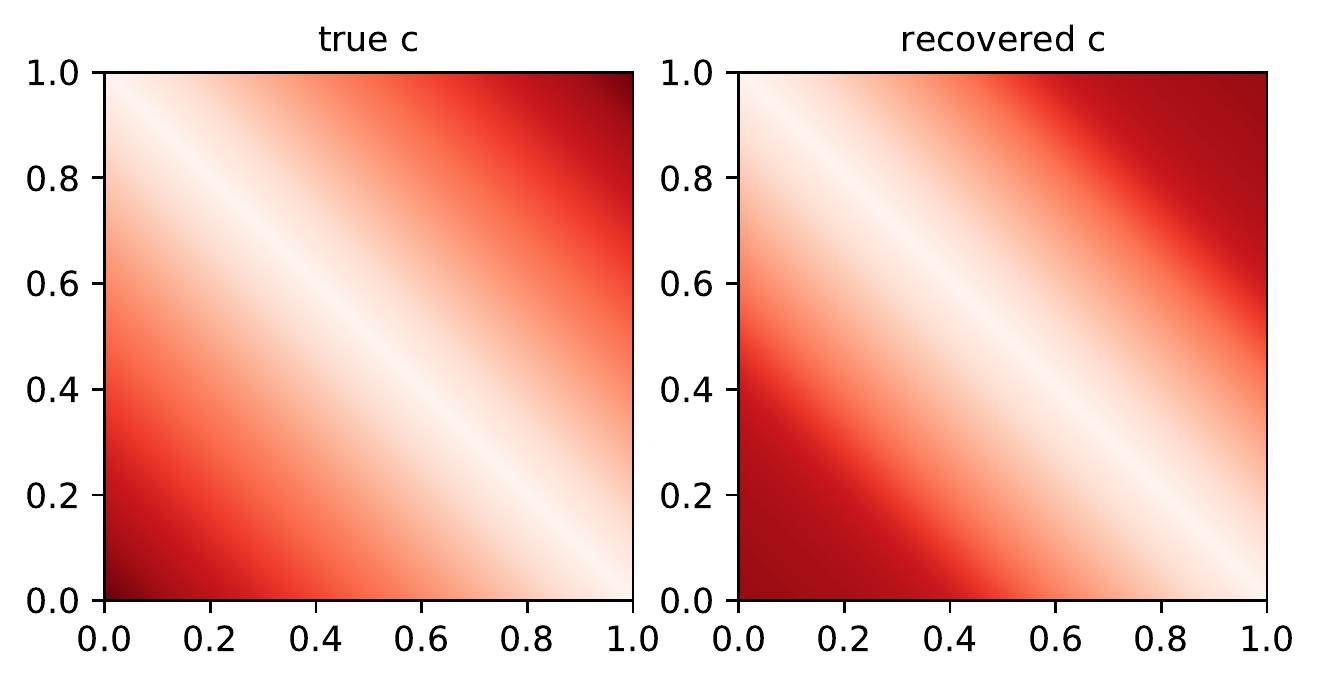} 
\caption{$|x-y|^{1}$}
\label{subfig:cts_c_rec_L1}
\end{subfigure}
\begin{subfigure}[b]{.48\columnwidth}
\includegraphics[width=\linewidth]{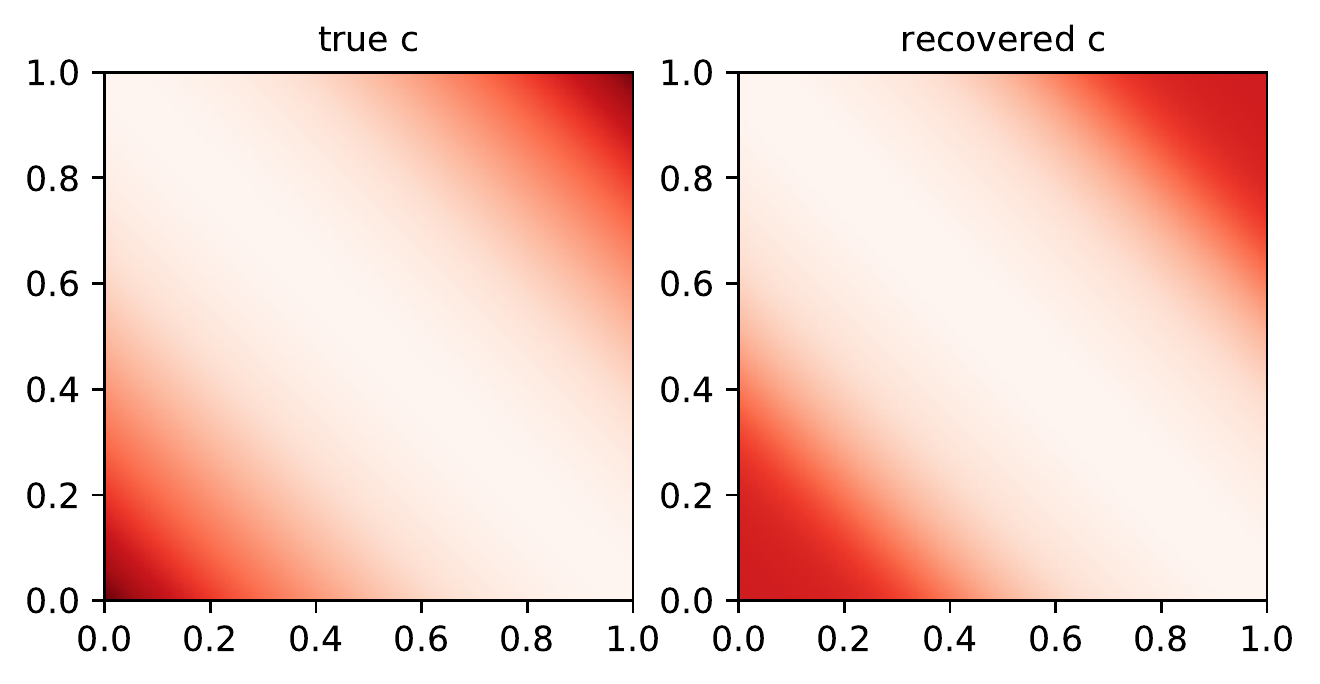} 
\caption{$|x-y|^{2}$}
\label{subfig:cts_c_rec_L2}
\end{subfigure}
\begin{subfigure}[b]{.48\columnwidth}
\includegraphics[width=\linewidth]{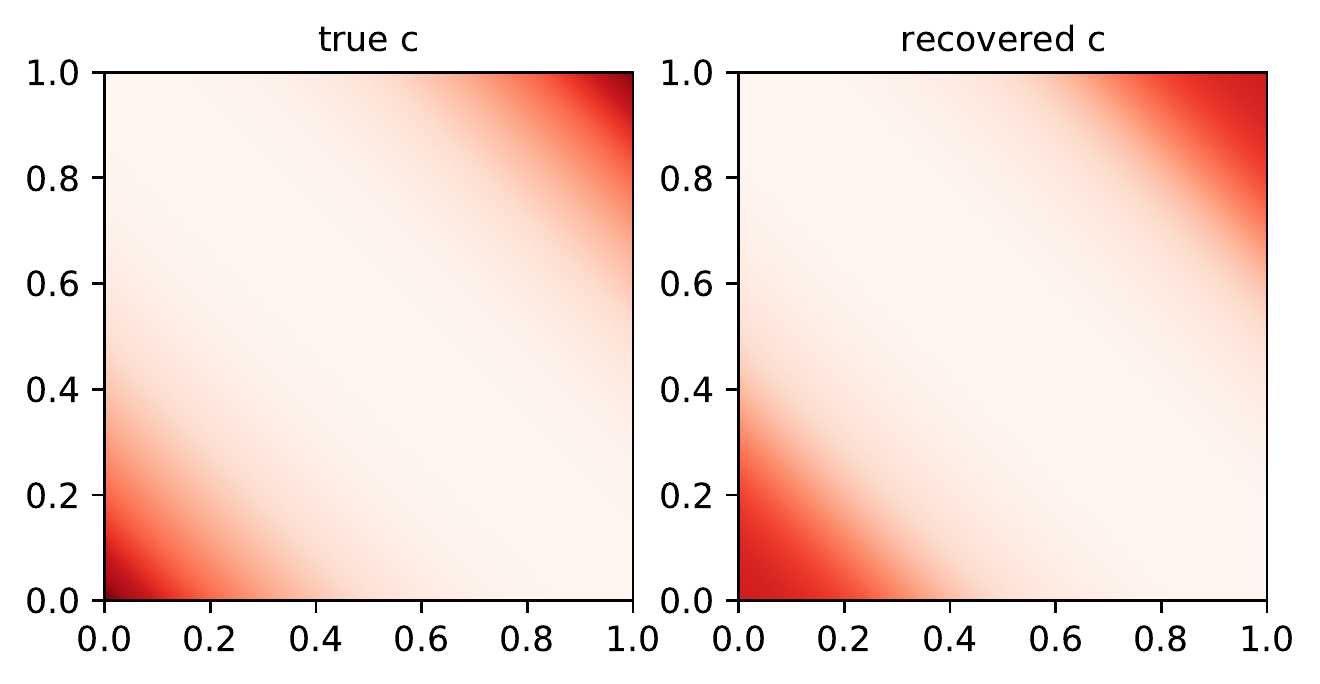} 
\caption{$|x-y|^{3}$}
\label{subfig:cts_c_rec_L3}
\end{subfigure}
\begin{subfigure}[b]{.48\columnwidth}
\includegraphics[width=\linewidth]{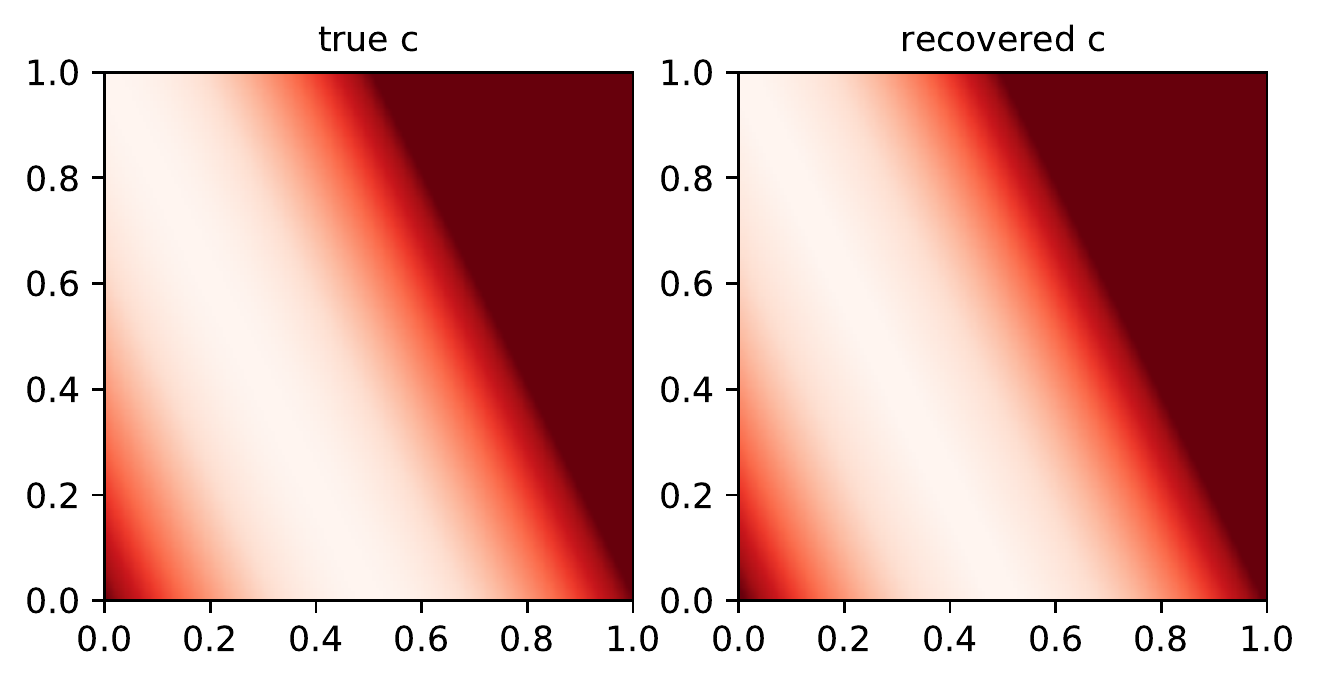} 
\caption{$|x-2y|^{2}$}
\label{subfig:cts_c_rec_L4}
\end{subfigure}
\begin{subfigure}[b]{.48\columnwidth}
\includegraphics[width=\linewidth]{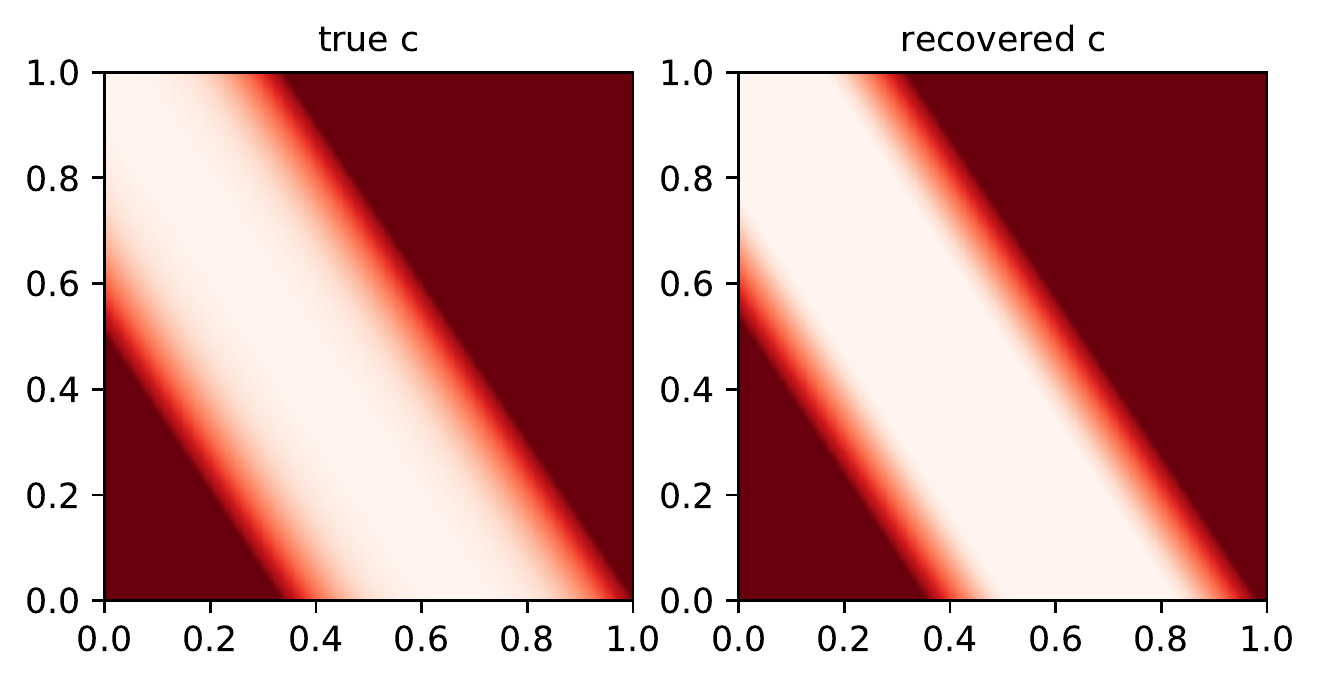} 
\caption{$|2x-3y|^{3}$}
\label{subfig:cts_c_rec_L5}
\end{subfigure}
%
%
\caption{True cost function and cost function recovered by Algorithm \ref{alg:cts_irot} assuming knowledge of the linear proportion between $x$ and $y$ for continuous inverse OT on synthetic data}
\label{fig:cts_c_rec_sym}
\end{figure}

We now apply Algorithm \ref{alg:cts_irot} to recover the cost function $c$ in continuous inverse OT.
The main difference from the discrete inverse OT is that, instead of learning a cost matrix, we aim at learning a parameterized function $c: X \times Y \to \mathbb{R}$ where $X \subset \mathbb{R}^{d_1}$ and $Y \subset \mathbb{R}^{d_2}$. Here $d_1$ and $d_2$ can be $3$ or even higher, which causes the issue known as the curse of dimensionality if we discretize $X$ and $Y$. 
In this case, we parameterize $c$ as a deep neural network, with input layer size $d_1+d_2$ and output layer size $1$, to overcome the issue of discretization in high-dimensional spaces. 
For simplicity, we consider the case where $d_1 = d_2$, but the method can be applied to general cases easily.

To justify the accuracy, we first consider the case with $d_1 = d_2 = 1$ so that we can discretize the problem and compute the ground truth optimal transport plan $\pidata$ accurately for sampling and evaluation purposes. 
We create a data set $\Dcal_{\pidata}$ by drawing $N=5,000$ samples from $\pidata$ and use them as the sample pairing data for cost learning in each iteration of Algorithm \ref{alg:cts_irot}.
We parameterize $c$ as a 5-layer (including one input layer, 3 hidden layers, and one output layer) deep neural network with 20 neurons per hidden layer, with tanh as the activation functions for the hidden layers and ReLU as the output layer.

In the first test, we set the cost to $c=|x-y|^p$ where $p=0.5,1,2,3$. Here we aim at learning the correct exponent function $(\cdot)^p$ and hence use $|x-y|$ instead of $(x,y)$ as the input (input layer dimension is 1 here). We use PyTorch \citep{paszke2017automatic} and the builtin ADAM optimizer \citep{kingma2014adam} with learning rate $10^{-4}$ for training the network $c$, where the parameters are initialized using Xavier initialization \citep{xvaier}.
%
%
%
%
\begin{figure}[t!]
\centering
\begin{subfigure}[b]{.48\columnwidth}
\includegraphics[width=\linewidth]{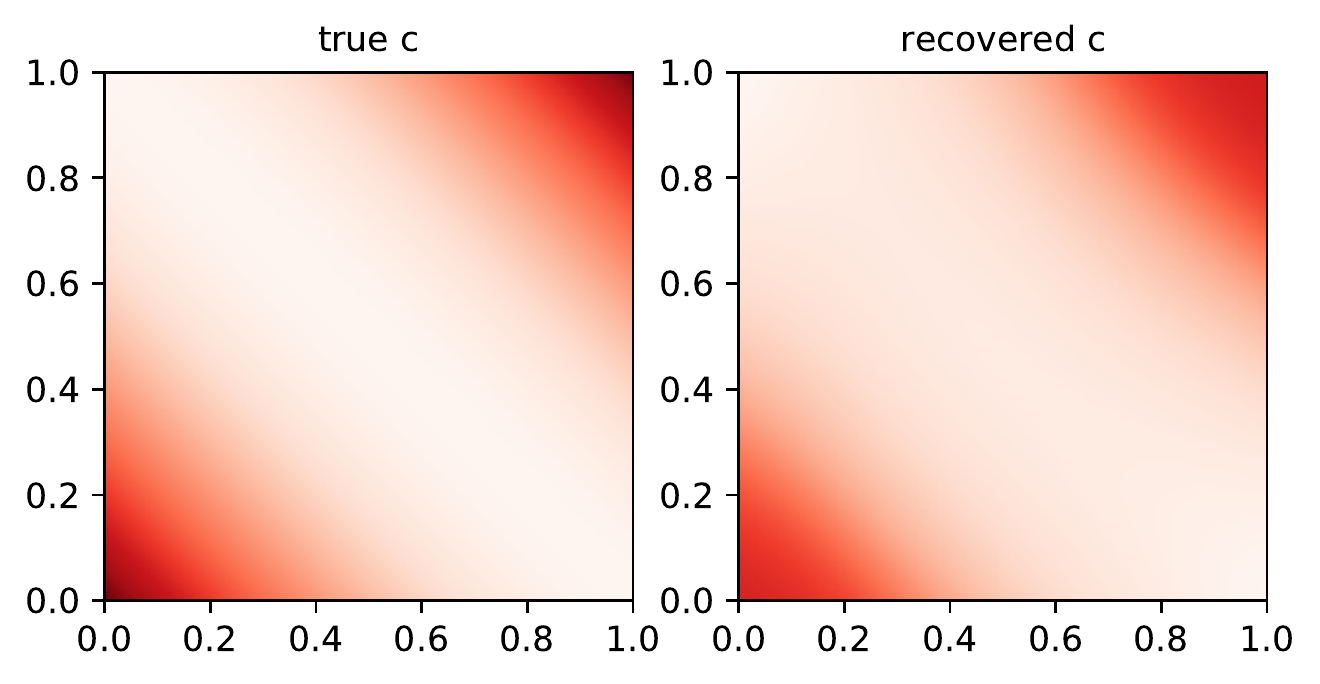} 
\caption{$|x-y|^{2}$}
\label{subfig:cts_c_rec_L6}
\end{subfigure}
\begin{subfigure}[b]{.48\columnwidth}
\includegraphics[width=\linewidth]{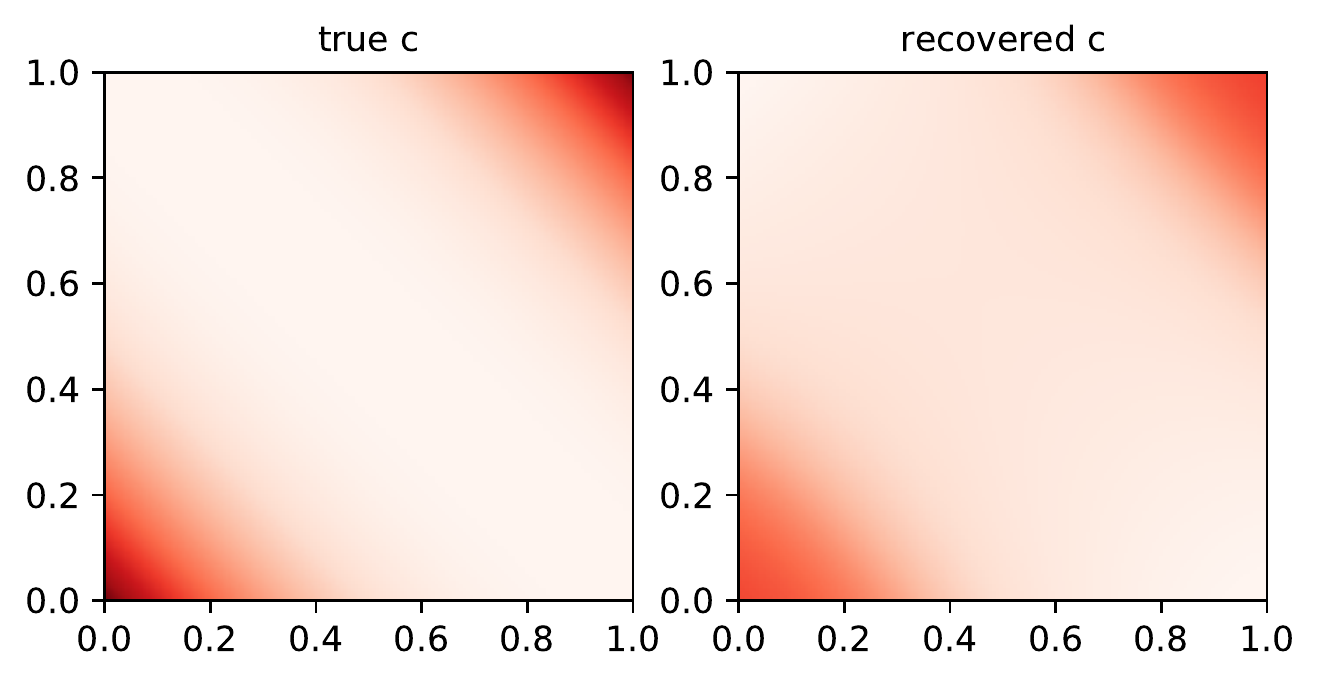} 
\caption{$|x-y|^{3}$}
\label{subfig:cts_c_rec_L7}
\end{subfigure}
%
\begin{subfigure}[b]{.48\columnwidth}
\includegraphics[width=\linewidth]{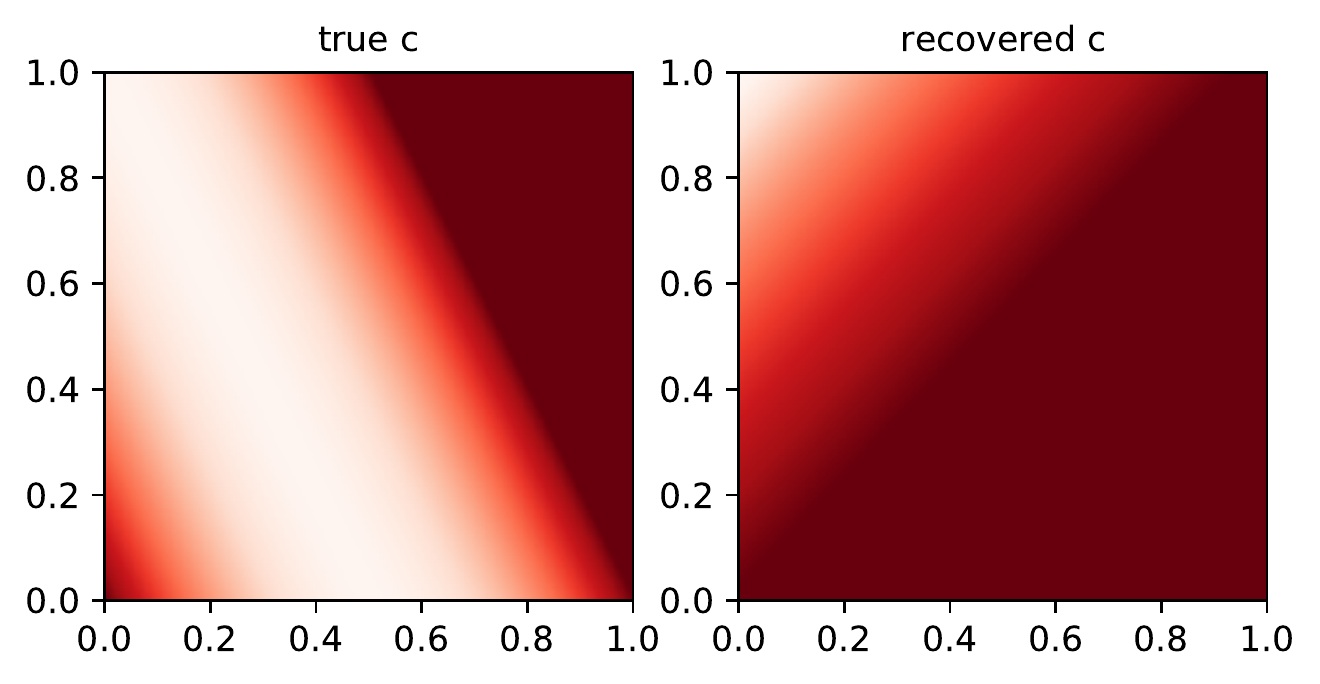} 
\caption{$|x-2y|^{2}$}
\label{subfig:cts_c_rec_L12}
\end{subfigure}
\begin{subfigure}[b]{.48\columnwidth}
\includegraphics[width=\linewidth]{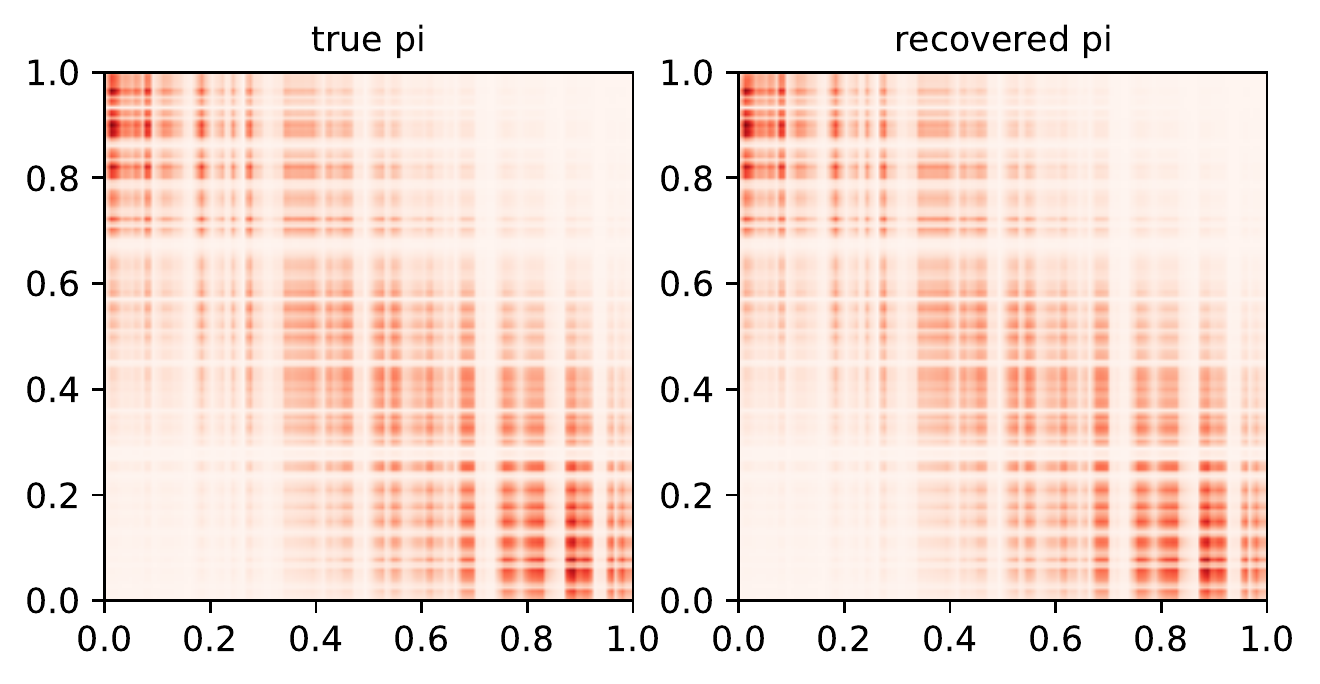}
\caption{$|x-2y|^{2}$, relative error of $\pi: 0.026$}
\label{subfig:cts_c_rec_L13}
\end{subfigure}
\caption{True cost function and cost function recovered by Algorithm \ref{alg:cts_irot} without knowledge of the proportion between $x$ and $y$ for continuous inverse OT on synthetic data. Notice that (d) shows the optimal transport plan induced by the true cost $c$ (left) and the one by the recovered cost (right) are very similar (with relative error 0.026) despite that the recovered cost differs significantly from the true cost shown in (c). This demonstrates the generic solution non-uniqueness issue of inverse OT if prior knowledge on $c$ is insufficient.}
\label{fig:cts_c_rec_sym_usual_input}
\end{figure}

In Figure \ref{fig:cts_c_plot}, we plot the progress of the relative error $\|c - c^*\|_F / \|c^*\|_F$ and objective function value versus iteration number using Algorithm \ref{alg:cts_irot}. These two plots indicate that both errors of the recovered $c$ and the objective function values obtained Algorithm \ref{alg:cts_irot} decay stably.
The learned cost functions (image in the right panel) are shown in Figure \ref{fig:cts_c_rec_sym} (a)--(d) for $p=0.5,1,2,3$ respectively, from which we can see that they match the ground truth cost functions (image in the left panel) closely.

We also consider a more challenging problem of recovering asymmetric cost functions $c^*(x,y) = |x-2y|^2$ and $c^*(x,y) = |2x-3y|^3$. We set the input as $\xi=|x-2y|$ and $\xi=|2-3y|$ for the cost function $c$ and again try to recover the unknown exponent $(\cdot)^p$.
The network structure and activation functions are set identically to the symmetric case.
The ground truth cost and learned cost functions are shown in Figure \ref{subfig:cts_c_rec_L4} and \ref{subfig:cts_c_rec_L5}, which demonstrate that Algorithm \ref{alg:cts_irot} can also faithfully learn the exponents in the asymmetric case. 

Now we conduct a test of Algorithm \ref{alg:cts_irot} without any prior information about the cost function. We set the ground truth cost function $c(x,y)$ to be $|x-y|^2$ and $|x-y|^3$, and use the same generic neural network $c:\mathbb{R}\times \mathbb{R} \to \mathbb{R}$ (3 hidden layer, 20 neurons per layer, and tanh and ReLU as the hidden layer activation and output activation respectively). The recovered cost functions are plotted in Figure \ref{subfig:cts_c_rec_L6} and \ref{subfig:cts_c_rec_L7}. From \ref{subfig:cts_c_rec_L6} and \ref{subfig:cts_c_rec_L7}, we see that Algorithm \ref{alg:cts_irot} can still recover the correct cost due to the symmetry.

We again test an asymmetric cost function $c(x,y) = |x-2y|^2$. We parameterize $c(x,y) = (x - \alpha y)^p$ where both $\alpha$ and $p$ are unknown. Even with such rich prior information about the cost function $c$, it is difficult to recover the ground truth $c$ faithfully. To see this, we plot the cost function recovered by Algorithm \ref{alg:cts_irot} and compare it with the true one in Figure \ref{subfig:cts_c_rec_L12}. As we can see, the two cost functions are very different. However, when we apply forward OT using these two cost functions, we obtain very similar transport plans with a small relative error 0.026, as shown in Figure \ref{subfig:cts_c_rec_L13}. This demonstrates the genuine difficulty in the inverse problem of OT: there can be a large number of cost functions that yield the same optimal transport plan as the given one, and it is critically important to impose proper restrictions to $c$ in order to recover the true cost function. Although we proved that this issue can be completely resolved with a mild assumption on the symmetry of $c$, it still can be a challenging issue in the most general case when such assumption does not hold.

\begin{figure}[t!]
\centering
\begin{subfigure}{.23\columnwidth}
\includegraphics[width=\linewidth]{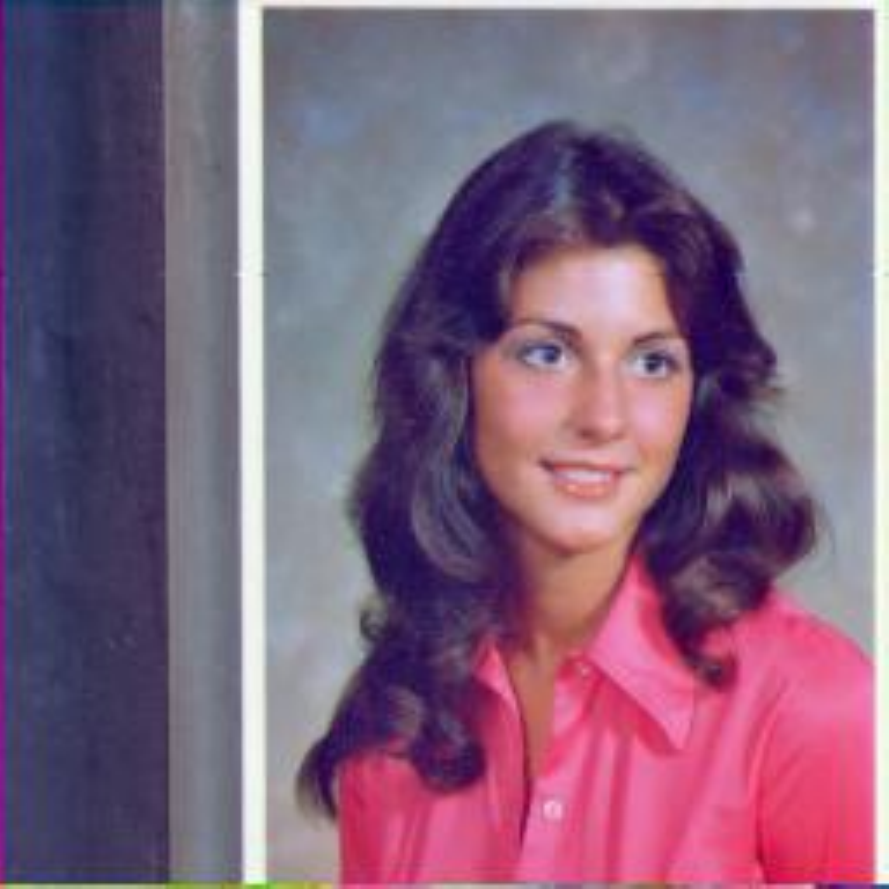} 
\caption{}
\label{subfig:cts_mar1}
\end{subfigure}
\begin{subfigure}{.23\columnwidth}
\includegraphics[width=\linewidth]{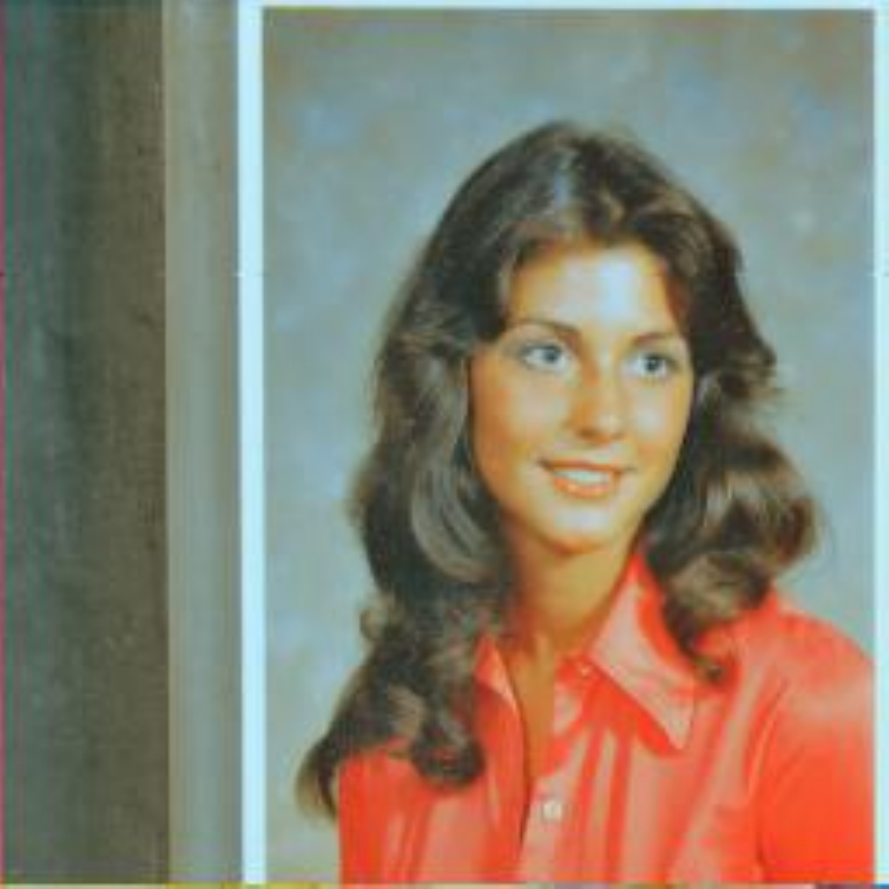} 
\caption{}
\label{subfig:cts_trans}
\end{subfigure}
\begin{subfigure}{.23\columnwidth}
\includegraphics[width=\linewidth]{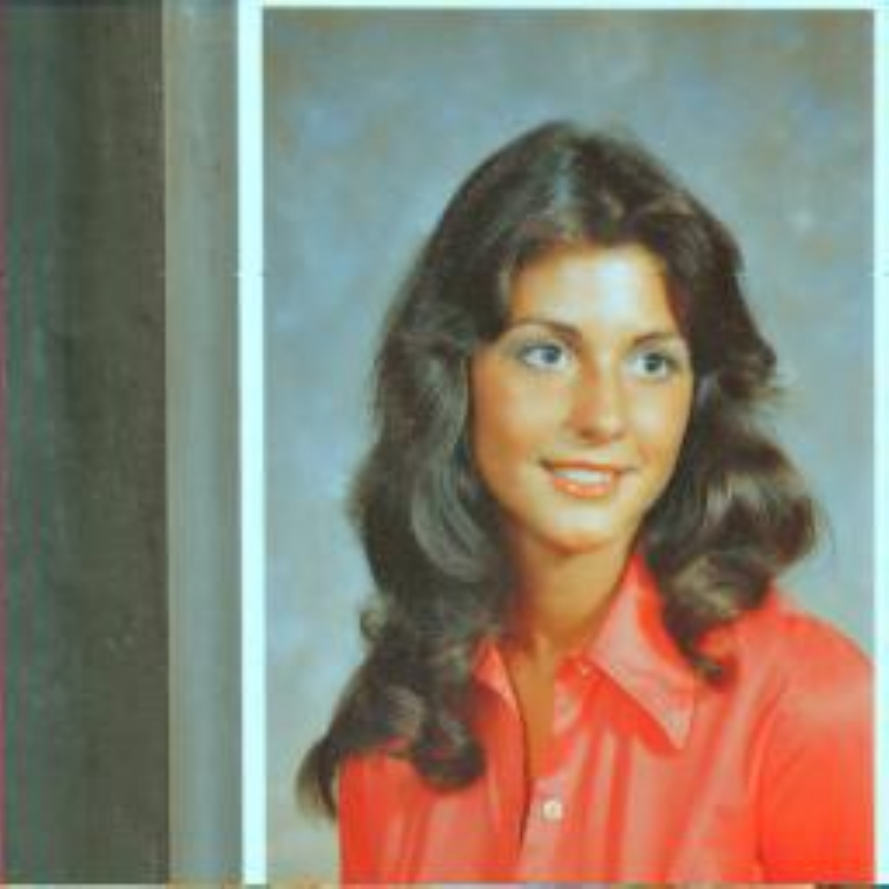} 
\caption{}
\label{subfig:cts_learned_img}
\end{subfigure}
\begin{subfigure}{.23\columnwidth}
\includegraphics[width=\linewidth]{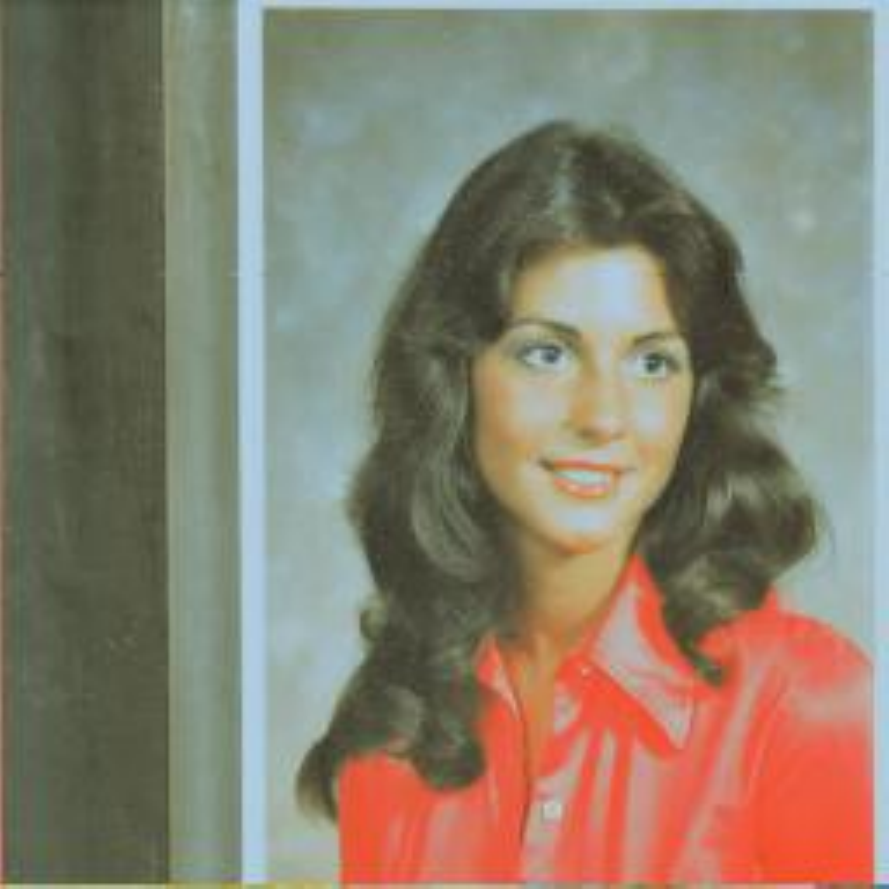} 
\caption{}
\label{subfig:cts_miss_img}
\end{subfigure}
\begin{subfigure}{.23\columnwidth}
\includegraphics[width=\linewidth]{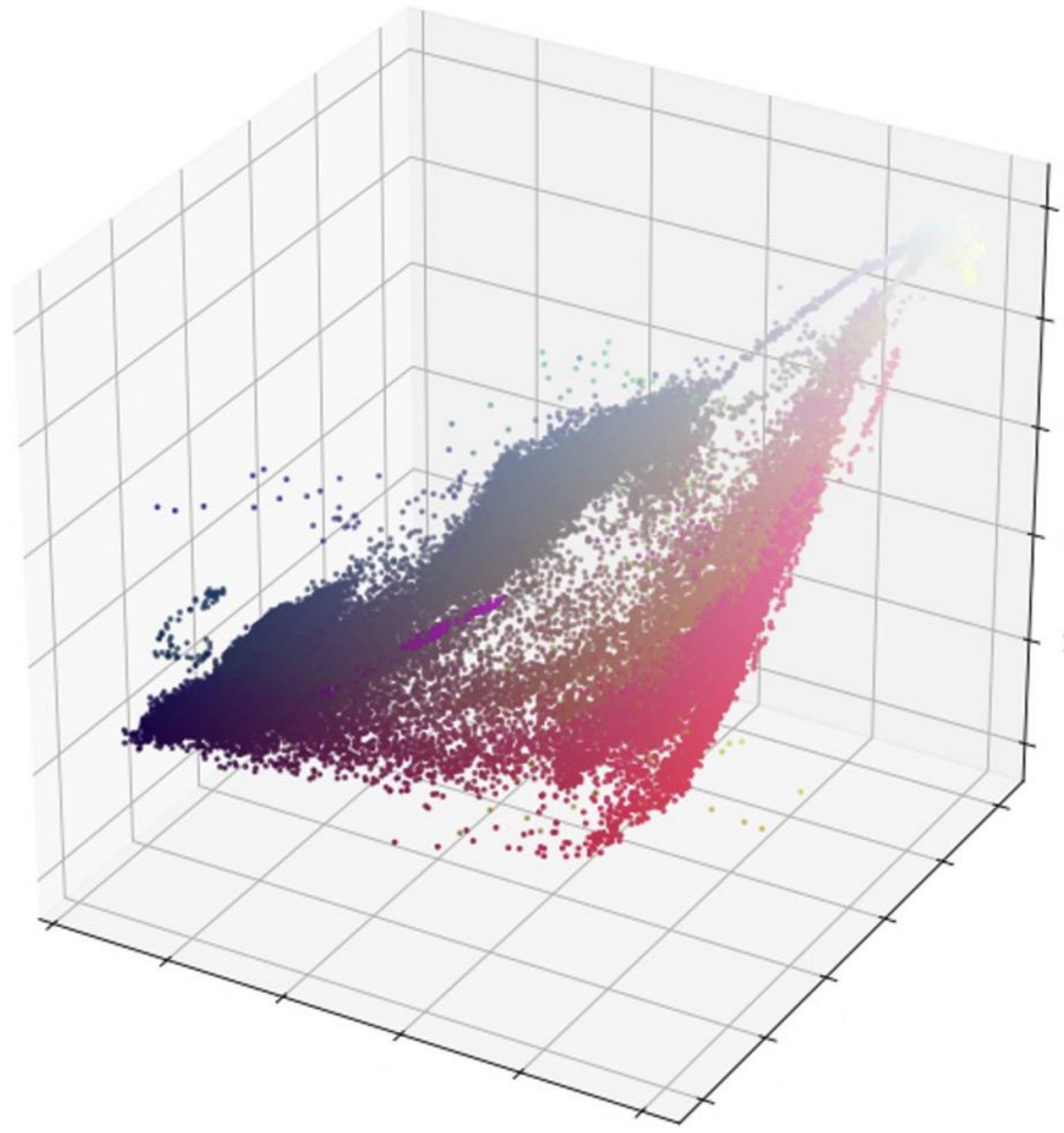} 
\caption{}
\label{subfig:cts_mar1_cl}
\end{subfigure}
\begin{subfigure}{.23\columnwidth}
\includegraphics[width=\linewidth]{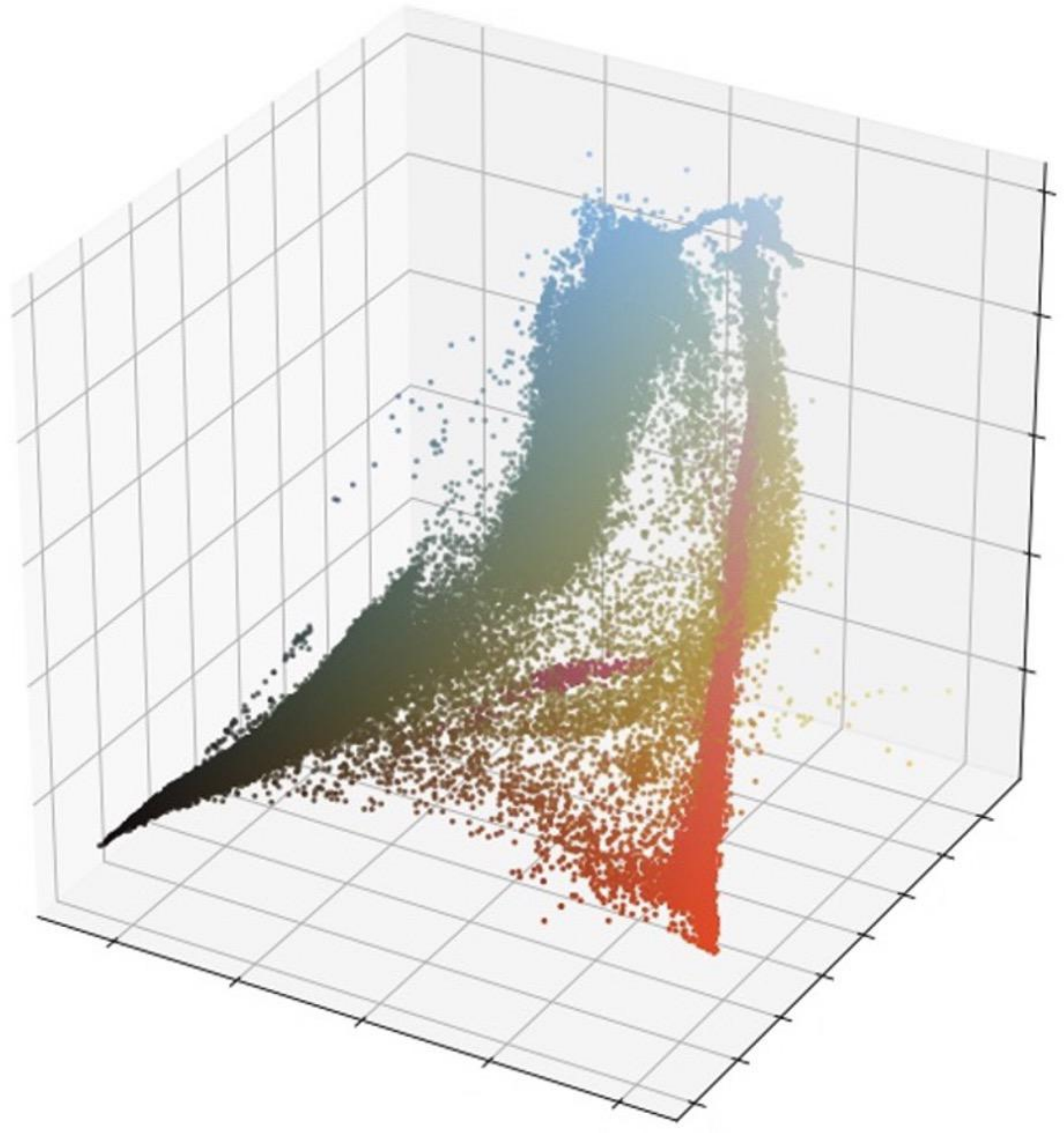} 
\caption{}
\label{subfig:cts_transcl}
\end{subfigure}
\begin{subfigure}{.23\columnwidth}
\includegraphics[width=\linewidth]{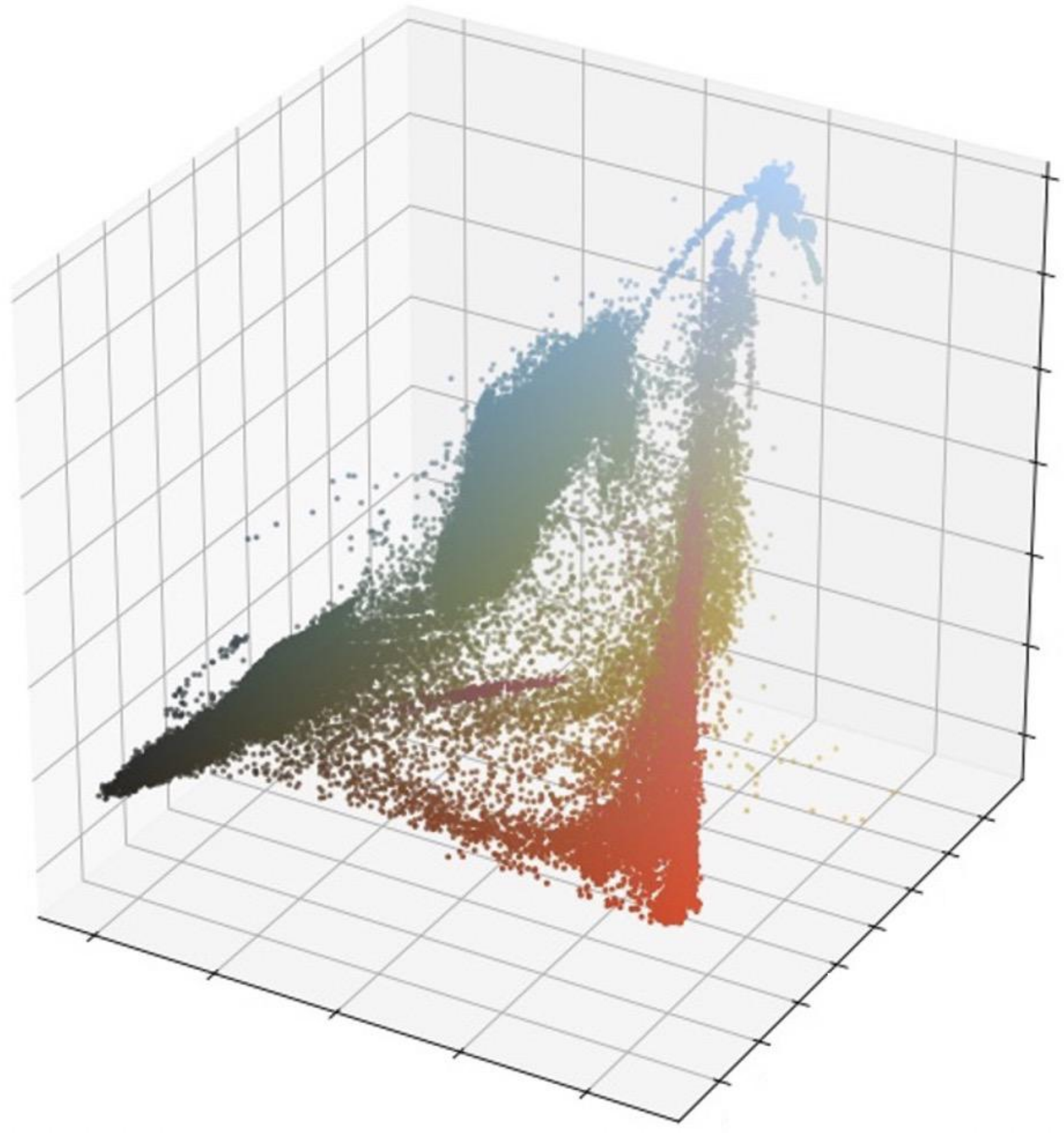}
\caption{}
\label{subfig:cts_learndcl}
\end{subfigure}
\begin{subfigure}{.23\columnwidth}
\includegraphics[width=\linewidth]{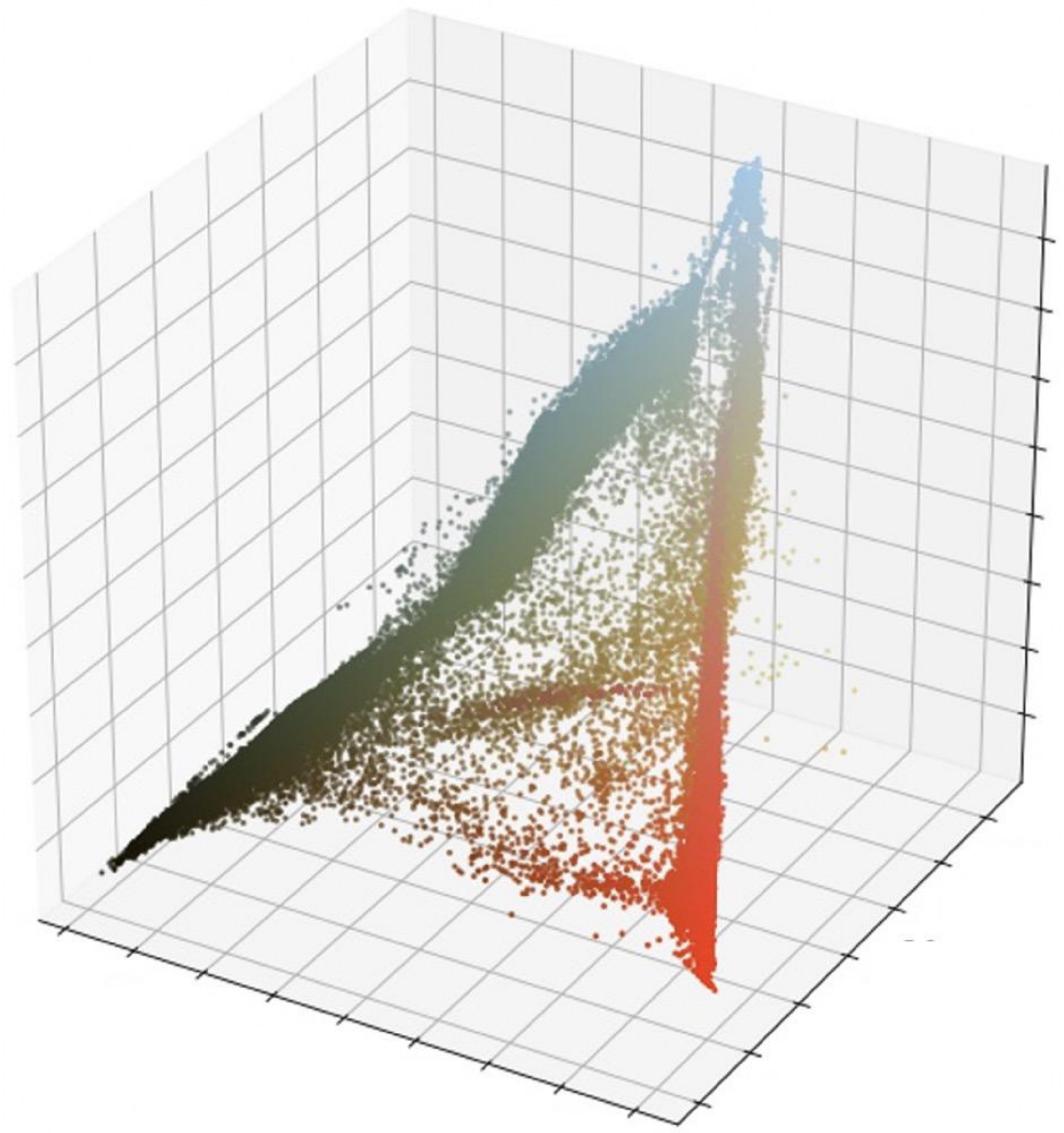} 
\caption{}
\label{subfig:cts_misscl}
\end{subfigure}
\caption{Result of color transfer using the cost funciton learned by Algorithm \ref{alg:cts_irot}. (a) Source image; (b) Target color transferred image; (c) Color transferred image using the cost function learned by Algorithm \ref{alg:cts_irot}; (d) Color transferred image using mis-specified cost function. Images in the bottom row show the point clouds of color pixels of the images above. The color in image (c) is much more faithful to (b), whereas (d) renders noticeable bias in color fading.}
\label{fig:cts_real}
\end{figure}

\subsection{Continuous Inverse OT on Color Transfer}
In this test, we consider the inverse OT problem in color transfer between images. Given two RGB color images, the goal of color transfer is to impose the color palette of one image (target) onto the other (source). It is natural to use 3D points to encode the RGB color of pixels, then each image can be viewed as a point cloud in $\mathbb{R}^3$, thus forming a pairing data with $d_X=d_Y=3$ using optimal transport under certain ground cost $c$. Specifically, given a cost function $c$, we can learn the pairing that transfers the point cloud of the source image to the one of the target image by solving a forward OT problem \citep{seguy2018large}.
However, the cost function is critical in shaping the the color transfer result. This experiment is to show how an adaptively learned cost using Algorithm \ref{alg:cts_irot} can help to overcome the issue with mis-specified cost and avoid inaccurate color transfer.


We obtain a pair of source and target images the USC-SIPI image database Volumn 3 \citep{uscsipi}. We set the ground truth cost as $\|x-y\|^2$, and follow \citep{seguy2018large} to generate the color transfer map. The original and the color transferred images are shown in Figure \ref{fig:cts_real}. The pairing of point clouds of these two images are used as the samples of $\pidata$ and fed into Algorithm \ref{alg:cts_irot}. In Algorithm \ref{alg:cts_irot}, we parameterize the cost function in the form of $c(x,y) = g((|x_1-y_1|, |x_2-y_2|, |x_3 - y_3|))$, where $g: \mathbb{R}^3 \to \mathbb{R}$ is a 5-layer neural network (including 3 hidden layers) with 32 neurons per hidden layer. The activation function is set to tanh. We use ADAM optimizer with learning rate $10^{-3}$. After the cost $c$ is learned using Algorithm \ref{alg:cts_irot}, we test the effect of the learned cost by applying it to the two given point clouds, and show the color transferred image using this learned cost in Figure \ref{subfig:cts_learned_img} and \ref{subfig:cts_learndcl}. For comparison, we also use a mis-specified cost function $c(x,y) = \|x - y\|$ to generate another color transferred image, as shown in Figure \ref{subfig:cts_miss_img} and its point cloud as \ref{subfig:cts_misscl}. As we can see, the image obtained using the cost learned by Algorithm \ref{alg:cts_irot} (Figure \ref{subfig:cts_learned_img}) is much more faithful to the true color (Figure \ref{subfig:cts_mar1}), whereas a mis-specified cost function yields an image (Figure \ref{subfig:cts_miss_img}) with clearly noticeable bias in color tone.

\section{Conclusion}
\label{sec:conclusion}
In this paper, we conduct a comprehensive study of the inverse problem for OT, i.e., learning the cost function given transport plan observations. We propose a novel inverse OT approach to learn the cost functions such that the induced OT plan is close to the observed plan or its samples. Unlike the bi-level optimization in the literature, we derive a novel formulation to learn the cost function by minimizing an unconstrained convex functional, which can be further augmented by customizable regularization on the cost. We provide a comprehensive characterization of the inverse problem, including the structure of its solution and mild conditions that yield solution uniqueness. We also developed two prototype numerical algorithms to recover the cost in the discrete and continuous settings separately. Numerical results show very promising efficiency and accuracy of our approach.


\newpage 

\appendix

\newpage
\section{Block Coordinate Descent for Discrete Inverse OT}
\label{appsec:bcd}

In Section \ref{subsec:dis_iot}, we mentioned that block coordinate descent (BCD) \citep{beck2015convergence} is a widely used approach for solving convex minimization with multiple variables, such as \eqref{eq:dis_irot_final}, when the closed form solution to subproblems are available.
However, convergence of BCD requires two key assumptions: the gradient of the objective function is Lipschitz continuous and the iterates generated by BCD are bounded. However, neither of these two assumptions holds for \eqref{eq:dis_irot_final}. 
To overcome these issues and ensure convergence of BCD, we can reformulate the minimization problem \eqref{eq:dis_irot_final} into an equivalent form, and show that the convergence can be guaranteed for this new variant of Algorithm \ref{alg:dis_irot}. This variant of BCD for \eqref{eq:dis_irot_final} is summarized in Algorithm \ref{alg:bcm}. The reformulation and its equivalency to \eqref{eq:dis_irot_final} is given in the following lemma.
\begin{lemma}
\label{lem:equiv}
The inverse OT minimization \eqref{eq:dis_irot_final} is equivalent to the following minimization:
\begin{equation} \label{eq:equiv}
    \min_{\alpha,\beta, c}\ \Psi(\alpha,\beta,c): = F(\alpha,\beta,c) + R(c),
\end{equation}
where the function $E(\alpha,\beta,c)$ in \eqref{eq:dis_irot_final} is replaced with
\[
F(\alpha,\beta,c) := - \langle \alpha, \mu \rangle - \langle \beta, \nu \rangle + \langle c, \pidata \rangle  + \varepsilon \log(\langle e^{(\alpha + \beta - c)/\varepsilon}, 1\rangle).
\]
The equivalence is in the sense that \eqref{eq:dis_irot_final} and \eqref{eq:equiv} share exactly the same set of solutions.
\end{lemma}
\begin{proof}
For any fixed $c$, we introduce Lagrange multipliers $\alpha$, $\beta$ and $\gamma$ for the equality constraints $\pi1 = \mu$, $\pi^{\top}1 = \nu$, $1^{\top}\pi 1 = 1$ respectively. Then we can form the dual problem of the entropy regularized OT:
\[
\max_{\alpha,\beta} \ \langle \alpha, \mu \rangle + \langle \beta, \nu \rangle - \varepsilon \log(\langle e^{(\alpha + \beta - c)/\varepsilon}, 1\rangle).
\]
The other parts of \eqref{eq:dis_irot_final} remain the same.
Then it is easy to verify that all statements of Theorem \ref{thm:convex} (for discrete setting here) still hold true. Hence \eqref{eq:dis_irot_final} and \eqref{eq:equiv} are equivalent. We omit the details here.
\end{proof}

The main advantages of \eqref{eq:equiv} are that the function $F$ is still smooth and convex in $(\alpha,\beta,c)$, the minimization subproblems (of $\alpha$ and $\beta$) still have closed form solutions, and that $\nabla_{\alpha} F, \nabla_{\beta} F$, and $\nabla_{c} F$ are all 1-Lipschitz continuous. The Lipschitiz continuity is a consequence of the following lemma.
\begin{lemma} \label{lem:lipschitz}
For any $a\in \mathbb{R}^n$ and $b \in \mathbb{R}_{+}^n$, the function $f(x):=\langle a, x\rangle + \log(\sum_{i=1}^n b_i e^{x_i})$ is convex in $x$, and $\nabla f$ is 1-Lipschitz.
\end{lemma}
\begin{proof}
It is straightforward to verify that $\partial_i f(x) = a_i + \frac{b_ie^{x_i}}{\sum_{j=1}^n b_j e^{x_j}}$. Furthermore, there is
\[
\partial_{ij}^2 f(x) = 
\begin{cases}
\frac{1}{\|\sqrt{w}\|_2^4}(\|\sqrt{w}\|_2^2 w_i - w_i^2), & \mbox{if } i = j, \\
-\frac{1}{\|\sqrt{w}\|_2^4}(w_iw_j), & \mbox{if } i \ne j,
\end{cases}
\]
where $w_i : = b_i e^{x_i}$ for $i=1,\dots,n$ and we adopted a slightly misused notation $\sqrt{w}:=(\sqrt{w_1},\dots,\sqrt{w_n})$. Then for any $\xi \in \mathbb{R}^n$, we can show that
\[
\xi^{\top} \nabla^2 f(x) \xi = \frac{1}{\|\sqrt{w}\|_2^4}( \|\sqrt{w}\|_2^2 \| \sqrt{w} \xi \|_2^2 - |\langle \sqrt{w}, \sqrt{w} \xi \rangle |^2 ),
\]
where $\sqrt{w}\xi := (\sqrt{w_1}\xi_1,\dots,\sqrt{w_n}\xi_n)$ stands for the componentwise product between $\sqrt{w}$ and $\xi$.
By Cauchy-Schwarz inequality, we have $\langle \sqrt{w}, \sqrt{w} \xi \rangle \le \|\sqrt{w}\|_2 \cdot \|\sqrt{w}\xi\|_2$, from which it is clear that $\xi^{\top} \nabla f(x) \xi \ge 0$. Hence $f$ is convex in $x$.
Furthermore, there is $\| \sqrt{w} \xi \|_2^2 \le  \| \sqrt{w}\|_2^2 \cdot \|  \xi \|_2^2$, from which we can see that 
\[
\xi^{\top} \nabla^2 f(x) \xi \le \frac{\|\sqrt{w}\xi\|_2^2}{\|\sqrt{w}\|_2^2} \le \|\xi\|_2^2,
\]
which implies that $\nabla f$ is 1-Lipschitz.
\end{proof}

To apply BCD with guaranteed convergence, we also need the boundedness of the iterates $x = (\alpha_k,\beta_k,c_k)$. 
In the literature of BCD or alternating minimization (AM), an assumption on the boundedness of the sub-level set $\{x: \Psi(x) \le \Psi(x_0)\}$ or that $\Psi$ is coercive is needed. However, neither of these holds for the \eqref{eq:equiv}. 
To ensure boundedness of the iterates, we can restrict our search of the cost matrix $c$ such that $0 \le c_{ij} \le M_c$ for some $M_c>0$ in addition to the constraint or regularization enforced by $R(c)$.
However, we do not have similar bounded restrictions on $\alpha$ and $\beta$.
To overcome this issue, we need to shift the solution of each minimization subproblem of BCD for \eqref{eq:equiv} without affecting its optimality.
To this end, we need the following definition.
\begin{definition}
A set $S \subset \mathbb{R}^n$ is said to have \emph{bounded variation} $M \in[0,\infty)$ if \[
\sup_{x\in S}\max_{1\le i,j \le n}|x_i - x_j| \le M.
\]
\end{definition}
Note that the requirement of bounded variation of $S$ is weaker than the boundedness of $S$.
Now we can show that the solution sets of the minimization subproblems in $\alpha$ and $\beta$ both have bounded variations in the following lemma. Note that we can always eliminate the zero components of $\mu$ and $\nu$ and regard them as strictly positive probability vectors.
\begin{lemma}
\label{lem:variation}
For any $c$ and any $\beta$, the set $\argmin_{\alpha} F(\alpha,\beta,c)$ has bounded variation $M_{\alpha}:=M_c + \varepsilon \log (\mu_{\max}/\mu_{\min})$. Similarly, for any $\alpha$, the set $\argmin_{\alpha} F(\alpha,\beta,c)$ has bounded variation $M_{\beta}:=M_c + \varepsilon \log (\nu_{\max}/\nu_{\min})$. Here $\mu_{\max}$ and $\mu_{\min}$ stand for the largest and smallest components of $\mu$ respectively. 
\end{lemma}
\begin{proof}
For any $c$ and $\beta$, we can check the optimality condition of an arbitrary $\alpha^* \in \argmin_{\alpha} F(\alpha,\beta,c)$. This condition is given by $\nabla_{\alpha} F(\alpha^*,\beta,c)=0$, which yields
\begin{equation}\label{eq:alpha_oc}
\frac{e^{\alpha_i^*/\varepsilon} \sum_{j} e^{(\beta_j - c_{ij})/\varepsilon}}{\sum_{i,j} e^{(\alpha^*_i +\beta_j - c_{ij})/\varepsilon}} = \mu_i.
\end{equation}
Taking logarithm of both sides and recalling the notation $s$ in \eqref{eq:def_s}, we obtain
\[
\frac{\alpha_i^*}{\varepsilon} = \log \mu_i + \log \varepsilon^{-1} s(\alpha^*,\beta,c) - \log \del[2]{ \sum_{j} e^{(\beta_j - c_{ij})/\varepsilon} }.
\]
Since $0\le c_{ij} \le M_c$, we know $e^{-M_c/\varepsilon} \le e^{-c_{ij}/\varepsilon} \le 1$, and hence from the equality above we obtain
\[
\varepsilon \log \del[2]{\frac{e^{\beta_j/\varepsilon}s(\alpha^*,\beta,c)}{\varepsilon}} + \varepsilon \log\mu_i \le 
\alpha_i^* \le
M_c +  \varepsilon \log \del[2]{\frac{e^{\beta_j/\varepsilon} s(\alpha^*,\beta,c)}{\varepsilon}} + \varepsilon \log\mu_i.
\]
Therefore the variation of $\alpha^*$, i.e., $\max_{1\le i, j \le n}|\alpha_i^* - \alpha_j^*|$, is bounded by $M_c + \varepsilon \log(\mu_{\max}/\mu_{\min})$. The proof for $\beta^*$ is similar and hence omitted.
\end{proof}

\begin{algorithm}[t]
\caption{Block Coordinate Descent (BCD) for Discrete Inverse OT \eqref{eq:dis_irot_final}}
\label{alg:bcm}
\begin{algorithmic}
\STATE {\bfseries Input:} Observed matching matrix $\hat{\pi} \in \mathbb{R}^{m\times n}$ and its marginals $\mu \in\mathbb{R}^m,\nu\in\mathbb{R}^n$. 
\STATE {\bfseries Initialize:} $\alpha \in \mathbb{R}^{m\times 1}, \beta\in\mathbb{R}^{n\times 1}, u=\exp(\alpha),v=\exp(\beta)$, $c\in\mathbb{R}^{m\times n}$, $c_{ij} \in [0,M_c]$.
\REPEAT
\STATE $K \gets e^{-c}$
\STATE $u \gets \mu/(K v)$ and rescale $u$ by $\kappa$ such that $e^{-M_\alpha} \le \lambda u \le e^{M_\alpha}$ 
\STATE $v \gets \nu/(K^{\top} u)$ and rescale $v$ by $\kappa$ such that $e^{-M_\beta} \le \kappa v \le e^{M_\beta}$  
\STATE $c \in \argmin_{0 \le c_{ij} \le M_c} R(c) + F(\log u, \log v, c)$ 
\UNTIL{convergent}
\STATE {\bfseries Output:} $\alpha = \varepsilon \log u$, $\beta= \varepsilon \log v$, $c = \varepsilon c$.
\end{algorithmic}
\end{algorithm}

Now we are ready to establish the convergence of Algorithm \ref{alg:bcm}.
For simplicity, we directly apply rescaling of $\alpha \leftarrow \alpha/\varepsilon, \beta \leftarrow \beta/\varepsilon, c \leftarrow c/\varepsilon$ which results in an equivalent problem of \eqref{eq:equiv} before Algorithm \ref{alg:bcm} starts, and rescale them back once the computation is finished.
Due to Lemma \ref{lem:variation}, we can always perform a shifting $\alpha \leftarrow \alpha - t 1$ such that $\|\alpha\|_{\infty} \le M_c/2$. The shifting constant $t \in \mathbb{R}$ can be simply set to $(\alpha_{(1)} - \alpha_{(n)})/2$, where $\alpha_{(1)}$ and $\alpha_{(n)}$ stand for the largest and smallest components of $\alpha$, respectively. As we can see, such shifting does not alter the optimality of $\alpha$ and it still satisfies \eqref{eq:alpha_oc}. Also note that this shifting is equivalent to rescaling $u=e^{\alpha}$ into $[e^{-M_c}, e^{M_c}]$ by $\kappa = e^t$, as presented in Algorithm \ref{alg:bcm}. The convergence of Algorithm \ref{alg:bcm} is given in the following theorem.

\begin{theorem}
\label{thm:bcm}
Let $(\alpha_k,\beta_k,c_k)$ be the sequence generated by the BCD Algorithm \ref{alg:bcm} from any initial $(\alpha_0,\beta_0,c_0)$, then
\begin{equation}\label{eq:bcm_rate}
0 \le \Psi_k - \Psi^* \le \min\cbr[2]{\frac{2}{9D^2}-2, 2, \Psi_0 - \Psi^*}\frac{18D^2}{k},
\end{equation}
where $\Psi_k:=\Psi(\alpha_k,\beta_k,c_k)$ and $D^2 = m M_{\alpha}^2 + n M_{\beta}^2 + mn M_c^2$.
\end{theorem}
\begin{proof}
By Lemma \ref{lem:lipschitz} we know $\nabla_{\beta}F(\alpha,\beta,c)$, $\nabla_{\alpha}F(\alpha,\beta,c)$ and $\nabla_{c}F(\alpha,\beta,c)$ are 1-Lipschitz continuous. Moreover, Algorithm \ref{alg:bcm} is equivalent to the standard BCD with where the iterates lie in the bounded set $\{(\alpha,\beta,c): \|\alpha\|_{\infty} \le M_\alpha,\ \|\beta\|_{\infty} \le M_\beta,\ \|c\|_{\infty} \le M_c\}$ due to Lemma \ref{lem:variation}. By invoking \cite[Theorem 2(a)]{hong2017iteration}, we obtain \eqref{eq:bcm_rate}.
\end{proof}

\vskip 0.2in
\bibliography{iot}

\end{document}